%% file: manuscript.tex
\documentclass[11pt,a4paper]{article}

\usepackage[T1]{fontenc}
\usepackage{amsmath,amssymb,amsfonts,amsthm}
\usepackage{textcomp}
\usepackage{booktabs}
\usepackage{graphicx}
\usepackage{algorithm}
\usepackage{algpseudocode}
\usepackage{xcolor}
\usepackage[colorlinks,citecolor=blue,urlcolor=blue,linkcolor=blue]{hyperref}
\usepackage{cleveref}
\usepackage{natbib}
\usepackage[margin=2.5cm]{geometry}
\usepackage{authblk}

\graphicspath{{./}}

\newtheorem{theorem}{Theorem}[section]
\newtheorem{corollary}[theorem]{Corollary}
\newtheorem{lemma}[theorem]{Lemma}
\theoremstyle{remark}
\newtheorem{remark}[theorem]{Remark}

\newcommand{\locc}{l_{\mathrm{occ}}}
\newcommand{\lfree}{l_{\mathrm{free}}}
\newcommand{\mO}{m_O}
\newcommand{\mF}{m_F}
\newcommand{\mOF}{m_{OF}}
\newcommand{\mOFmin}{m_{OF,\min}}
\newcommand{\mFobs}{m_{F,\mathrm{obs}}}
\newcommand{\BetP}{\mathrm{BetP}}
\newcommand{\DS}{\text{DS}}

\begin{document}

\title{Equivalence and Divergence of Bayesian Log-Odds and Dempster's Combination Rule for 2D Occupancy Grids}

\author[1]{Tatiana Berlenko}
\author[1,2]{Kirill Krinkin\thanks{Corresponding author: \texttt{kirill@krinkin.com}}}
\affil[1]{Constructor University, Bremen, Germany}
\affil[2]{JetBrains LTD}

\date{}
\maketitle

\begin{abstract}
Prior comparisons of Dempster's combination rule and Bayesian log-odds
for occupancy grid mapping have not controlled for sensor model
equivalence at the per-observation level, making it impossible to
attribute observed performance differences to the fusion rule rather
than to the sensor parameterization. We demonstrate that this
methodological gap produces large confounds: in our own multi-robot
experiment, unmatched parameterization reversed the boundary sharpness
comparison from $+12\%$ in favor of belief functions to $-22\%$
in favor of Bayesian log-odds.

We introduce a pignistic-transform-based methodology that matches
per-observation decision probabilities across frameworks, enabling
controlled comparison of Bayesian log-odds, Dempster's combination
rule, and Yager's rule in simulation (single-agent, multi-robot),
across sensor parameters and clamping limits, and on two real indoor
lidar datasets (Intel Research Lab, Freiburg Building~079).
Real-data multi-robot conditions use scan-splitting with ideal
alignment rather than true multi-robot deployment.

Under fair comparison with pignistic ($\BetP$) matching, Bayesian
fusion produces better point probabilities than Dempster's rule:
small absolute differences ($0.001$--$0.022$ on $[0,1]$ scales across
simulation and real data) are practically small.
Bayesian log-odds is directionally favored in 15/15 per-metric
simulation comparisons ($p = 3.1 \times 10^{-5}$); cell accuracy
differences are additionally TOST-equivalent at nominal margins.
Under normalized plausibility ($P_{\mathrm{Pl}}$) matching in
single-agent simulation, the direction reverses: DS achieves better
boundary sharpness (15/15) and Brier score (15/15), confirming that
the comparison is matching-criterion-specific. An $L_{\max}$ ablation
confirms the Bayesian advantage under $\BetP$ matching arises from
Dempster's $1/(1-K)$ conflict normalization, not from regularization.
A downstream A* path-planning evaluation confirms functionally
identical navigation outcomes. Previously reported Dempster--Shafer
advantages may be attributable, at least in part, to unmatched sensor
model parameterization.

The pignistic transform matching methodology is reusable for any
future comparison of Bayesian and belief function fusion. These
conclusions are restricted to point-probability metrics; the
interval-valued representation $[\mathrm{Bel}(A), \mathrm{Pl}(A)]$
unique to belief functions is not evaluated here.
\end{abstract}


\vspace{1em}

\input{introduction}
\input{background}
\input{equivalence}
\input{methodology}
\input{results}
\input{discussion}
\input{conclusion}

\input{appendix}

\bibliographystyle{plainnat}
\bibliography{references}

\end{document}


\title{Supplementary Material: Equivalence and Divergence of Bayesian Log-Odds and Dempster's Combination Rule for 2D Occupancy Grids}
\author{Tatiana Berlenko, Kirill Krinkin}
\date{}
\maketitle

\setcounter{section}{0}
\renewcommand{\thesection}{S\arabic{section}}
\renewcommand{\thesubsection}{S\arabic{section}.\arabic{subsection}}
\renewcommand{\theequation}{S\arabic{equation}}
\renewcommand{\thetable}{S\arabic{table}}

\section{Statistical Methodology}
\label{app:stats}

This supplement presents the complete statistical analysis framework
summarized in the main text.

\subsection{Directional Analysis}
\label{sec:directional}

As a distribution-free primary measure of effect reliability, we report
$k/N$: the number of runs in which the Bayesian arm achieves a
higher (or lower, for Brier score and entropy) metric value than the
belief function arm. This sign-test-like statistic requires no
assumptions about effect size distributions and is robust to outliers.

\subsection{Equivalence Testing (TOST)}
\label{sec:tost}

We apply the Two One-Sided Tests (TOST) procedure
\cite{schuirmann1987comparison,lakens2017equivalence} to assess whether
performance differences fall within practically negligible bounds.
TOST rejects the null hypothesis of non-equivalence when the
$(1 - 2\alpha)$ confidence interval for the mean paired difference
$\bar{d} = \overline{x - y}$ falls entirely within $(-\delta, +\delta)$,
where $\delta$ is the equivalence margin.

Rather than reporting equivalence at a single pre-specified margin, we
present a \textbf{sensitivity analysis} as the primary TOST result: for
each metric--condition pair, we sweep the margin from $0.5\times$ to
$2\times$ the nominal value and report the smallest margin at which
equivalence holds and the breakpoint at which it fails. This approach
provides strictly more information than a single-margin test and avoids
dependence on a potentially arbitrary margin choice.

As secondary support, we provide a practical argument for the nominal
margins: at $0.1$\,m grid resolution, a cell accuracy margin of
$\delta = 0.02$ corresponds to approximately 2~misclassified cells per
100~observed cells, which is negligible for downstream path planning
and obstacle avoidance at standard navigation tolerances.

\subsection{Effect Sizes}
\label{sec:effect-sizes}

Paired Cohen's $d$ is computed as $d = \bar{d} / s_d$, where $s_d$ is
the standard deviation of the paired differences. Confidence intervals
use the Hedges--Olkin \cite{hedges1985statistical} variance
approximation:
\begin{equation}
  \widehat{\mathrm{Var}}(d) = \frac{1}{n} + \frac{d^2}{2(n-1)},
  \label{eq:hedges-olkin-var}
\end{equation}
with 95\% CIs constructed via $d \pm t_{0.975, n-1} \cdot
\sqrt{\widehat{\mathrm{Var}}(d)}$.

For verification, we also compute exact confidence intervals via
non-central $t$ distribution inversion \cite{steiger2004beyond}: for
each observed $t$-statistic, the non-centrality parameters
$\delta_{\mathrm{lo}}$ and $\delta_{\mathrm{hi}}$ satisfying
$P(T > t_{\mathrm{obs}} \mid \delta_{\mathrm{lo}}) = \alpha/2$ and
$P(T < t_{\mathrm{obs}} \mid \delta_{\mathrm{hi}}) = \alpha/2$ are
found numerically, yielding exact CI bounds $d = \delta / \sqrt{n}$.
Across all nine comparisons, the maximum absolute discrepancy between
the Hedges--Olkin and non-central $t$ CI bounds is $0.59$, which is
less than 4\% of the corresponding point estimate. All qualitative
conclusions are identical under both methods.

\paragraph{Interpreting large $d$ values}
Cohen's \cite{cohen1988statistical} benchmarks ($d = 0.2$ small, $0.5$
medium, $0.8$ large) were calibrated for human behavioral studies with
high individual variability. In computational experiments with
controlled environments and fixed random seeds, within-run variance
approaches zero as the simulation becomes more deterministic, inflating
$d$ without limit. The effect sizes reported in this study ($d = 4.6$
to $14.9$) reflect very low within-method variance (standard deviation
of paired differences $s_d = 0.0001$--$0.012$), not large absolute
differences. The raw mean differences are $0.001$--$0.070$ on $[0,1]$
scales. We report both raw differences and $d$ values throughout, and
note that the $k/N$ directional consistency provides a more intuitive
measure of effect reliability in this setting.

\subsection{Multiplicity Correction}
\label{sec:multiplicity}

The nine simulation TOST comparisons (3~metrics $\times$ 3~conditions:
single-agent, dynamic baseline, noisy sensor) are
corrected as a single family using the Holm--Bonferroni step-down
procedure \cite{holm1979simple}, which controls the family-wise error
rate while being uniformly more powerful than the Bonferroni correction.
This is a conservative choice: correcting per-experiment (family size~3)
yields identical conclusions for all comparisons.

\subsection{Real-Data Statistical Analysis}
\label{sec:real-data-stats}

The real-data experiments (Intel Lab, Freiburg~079) produce a single
map per condition rather than independent replications. We therefore
use a \emph{spatial block bootstrap} to construct 95\% confidence
intervals that account for spatial autocorrelation in occupancy grids.

The grid is partitioned into non-overlapping $B \times B$ cell blocks
(default $B = 10$ cells $= 1.0$\,m at $0.1$\,m resolution). For each
metric, only blocks containing at least one evaluation cell are
retained. At each of $10{,}000$ bootstrap iterations, blocks are
resampled with replacement; the per-cell metric differences from all
cells in the resampled blocks are aggregated to compute the mean delta.
The 95\% percentile CI is then extracted from the bootstrap
distribution.

The block size of $1.0$\,m is chosen to exceed the grid cell size by a
factor of $10\times$, capturing local spatial dependence from
overlapping lidar rays. The autocorrelation range is expected to be less
than $2.0$\,m for indoor lidar grids at $0.1$\,m resolution, based on
the beam footprint ($0.1$--$0.5$\,m at typical indoor ranges) and the
resulting local correlation structure. A
sensitivity analysis over block sizes $B \in \{5, 10, 20\}$
($0.5$\,m, $1.0$\,m, $2.0$\,m) confirms that CI widths and
significance verdicts are stable across this range
(\cref{sec:appendix-block-sensitivity}).

\subsection{Bayes Factors}
\label{sec:appendix-bf}

As a complementary Bayesian measure, we compute $\mathrm{BF}_{01}$ for
each comparison using a conjugate normal model with a weakly
informative prior $\mu_d \sim \mathcal{N}(0, \delta^2)$ centered at
zero with scale equal to the equivalence margin. $\mathrm{BF}_{01} > 1$
favors equivalence; values exceeding $10^6$ indicate overwhelming
posterior concentration within the equivalence region.

Bayesian equivalence testing yields $\mathrm{BF}_{01} > 10^6$ for all
five TOST-equivalent comparisons, indicating overwhelming posterior
concentration within the equivalence region. The four non-equivalent
comparisons yield $\mathrm{BF}_{01} \approx 0$, indicating strong
evidence against equivalence. The BF results are fully consistent with
the frequentist TOST verdicts.
The prior scale equals the TOST margin, coupling the two analyses;
however, the overwhelming magnitude of $\mathrm{BF}_{01}$ ($> 10^6$)
ensures robustness to any reasonable prior specification.

\subsection{CI Verification}
\label{sec:appendix-ci}

\Cref{tab:ci-verification} reports the
verification of Cohen's $d$ confidence intervals: the
Hedges--Olkin approximation~\cite{hedges1985statistical}
used in Table~2 of the main paper is compared against exact
non-central $t$ distribution CIs~\cite{steiger2004beyond} for all nine
simulation comparisons. The maximum absolute discrepancy is $0.59$,
occurring for the single-agent Brier score comparison ($d = -14.93$,
the largest effect size; less than 4\% of the point estimate).
All qualitative conclusions are identical under both methods.

\input{table_ci_verification}

\subsection{Block Size Sensitivity for Spatial Block Bootstrap}
\label{sec:appendix-block-sensitivity}

\Cref{tab:block-sensitivity} reports CIs at $B \in \{5,
10, 20\}$ ($0.5$\,m, $1.0$\,m, $2.0$\,m) for both real-data
datasets. All significance verdicts are stable across block
sizes: no comparison changes from significant to non-significant or
vice versa.

\input{table_block_sensitivity}

\section{Supporting Proofs}
\label{app:proofs}

\subsection{Proof of Closed-Form Accumulation Lemma}

Under the conditions stated in the main text (vacuous prior, $N$
identical consonant occupied observations with
$m_{\mathrm{obs}} = (a,\, 0,\, 1-a)$):

\begin{proof}
By induction on $N$. The base case $N = 1$ is immediate from the
vacuous identity property. For the inductive step, assume
$\mO^{(N)} = 1 - (1-a)^N$ and $\mOF^{(N)} = (1-a)^N$
hold at step~$N$. Since $\mF^{(N)} = 0$ and
$\mFobs = 0$, the conflict is $K = 0$.
Applying Dempster's combination:
\begin{align*}
  \mO^{(N+1)}
    &= \mO^{(N)} \cdot a + \mO^{(N)} \cdot (1-a)
       + (1-a)^N \cdot a \\
    &= \mO^{(N)} + a(1-a)^N \\
    &= 1 - (1-a)^N + a(1-a)^N
     = 1 - (1-a)^{N+1}.
\end{align*}
Similarly,
$\mOF^{(N+1)} = (1-a)^N \cdot (1-a) / 1 = (1-a)^{N+1}$.
\end{proof}

\subsection{Multi-Observation Divergence}
\label{app:divergence}

The single-observation equivalence holds at each individual sensor
reading. Under repeated observations, the two accumulation mechanisms
diverge:
\begin{itemize}
  \item \textbf{Bayesian}: log-odds accumulate additively
    and are clamped at $|L| \leq L_{\max}$, yielding
    $L^{(N)} = \min(N \cdot \locc,\; L_{\max})$ for identical
    occupied observations.
  \item \textbf{Dempster--Shafer}: masses combine multiplicatively,
    with conflict normalization by $1/(1-K)$ when observations disagree.
\end{itemize}
The divergence arises from two distinct sources: (a)~the Bayesian
clamp imposes an artificial ceiling on confidence, while Dempster's
rule converges asymptotically to certainty without bound; and
(b)~Dempster's conflict normalization factor $1/(1-K)$ redistributes
mass in a manner that has no Bayesian analogue.

\section{Convergence Analysis}
\label{app:convergence}

\Cref{tab:convergence} compares the two accumulation trajectories
for $N$ identical occupied observations.

\begin{table}[t]
\centering
\caption{Convergence comparison for $N$ identical occupied observations
  ($\locc = 2.0$, $L_{\max} = 10$, $a = 0.7616$). At $N = 1$ the
  methods are identical by construction (pignistic matching). The
  Bayesian clamp saturates at $N = 5$; Dempster--Shafer overtakes at
  $N \approx 7$.}
\label{tab:convergence}
\begin{tabular}{@{}rcccccc@{}}
\toprule
$N$ & $L$ & $p(O)$ & $\mO$ & $\mOF$ & $\BetP(O)$ & $\BetP - p$ \\
\midrule
 1  &  2.0 & 0.8808 & 0.7616 & 0.2384  & 0.8808 & $+$0.0000 \\
 2  &  4.0 & 0.9820 & 0.9432 & 0.0568  & 0.9716 & $-$0.0104 \\
 5  & 10.0 & 0.9\hspace{-.2pt}99955 & 0.9992 & 0.0008  & 0.9996 & $-$0.0003 \\
10  & 10.0 & 0.9\hspace{-.2pt}99955 & $\approx 1$ & $< 10^{-6}$
    & $\approx 1$ & $+$0.00005 \\
20  & 10.0 & 0.9\hspace{-.2pt}99955 & $\approx 1$ & $< 10^{-12}$
    & $\approx 1$ & $+$0.00005 \\
\bottomrule
\end{tabular}
\end{table}

The methods agree at $N = 1$ (by construction) and then diverge: Bayesian
leads at $N = 2$--$5$ due to faster initial accumulation, clamps at
$N = 5$, and is overtaken by the Dempster--Shafer trajectory at
approximately $N = 7$.

\section{Ablation Studies}
\label{app:ablations}

\subsection{$L_{\max}$ Ablation}
\label{app:lmax}

We repeat the single-agent experiment at four clamping levels:
$L_{\max} \in \{5, 10, 20, \infty\}$, with 15 runs per configuration.

\begin{table}[t]
\centering
\caption{$L_{\max}$ ablation: Bayesian+count minus Dempster delta
(15 runs each). Higher is better for cell accuracy and boundary
sharpness; lower (more negative) is better for Brier score.}
\label{tab:lmax-ablation}
\begin{tabular}{lccccc}
\toprule
Metric & $L_{\max}=5$ & $L_{\max}=10$ & $L_{\max}=20$ & $L_{\max}=\infty$ & Consist. \\
\midrule
Cell accuracy      & $+0.0013$ & $+0.0013$ & $+0.0013$ & $+0.0013$ & 15/15 \\
Boundary sharpness & $+0.007$  & $+0.010$  & $+0.010$  & $+0.010$  & 15/15 \\
Brier score        & $-0.006$  & $-0.006$  & $-0.006$  & $-0.006$  & 15/15 \\
\bottomrule
\end{tabular}
\end{table}

The Bayesian advantage does not disappear at $L_{\max} = \infty$:
the gap persists with 15/15 directional consistency for all metrics.
Cell accuracy $\Delta$ is invariant ($+0.0013$) because it is
measured at the $p = 0.5$ classification threshold, far from the
saturation regime. Boundary sharpness $\Delta$ increases slightly
from $+0.007$ ($L_{\max} = 5$) to $+0.010$ ($L_{\max} = \infty$).

\subsection{Yager's Rule Comparison}
\label{app:yager}

Yager's combination rule~\cite{yager1987dempster} assigns conflict mass
to total ignorance ($\mOF$) rather than normalizing it away.
In simulation, Yager performs worse than Dempster on
boundary sharpness ($-24\%$ to $-31\%$ relative) because
conflict-to-ignorance transfer preserves excessive uncertainty at
boundary cells. Yager slightly outperforms Dempster on cell accuracy
($+0.025\%$) but performs worse on Brier score.

On Intel Lab data, Yager shows better Brier
score than Bayesian in the 1-source condition, and better cell
accuracy than Dempster in scan-split configurations. However, Yager
consistently produces the lowest boundary sharpness of all four arms.

\subsection{DS Regularization Control}
\label{app:ds-reg}

We introduce a minimum ignorance mass floor $\mOFmin$ for Dempster's
rule and test $\mOFmin \in \{0.001, 0.01, 0.05\}$.

The regularization does not close the gap. On cell accuracy, the gap
narrows modestly at $\mOFmin = 0.05$ ($6.9\%$ reduction). On
boundary sharpness, the gap \emph{widens}: at $\mOFmin = 0.05$,
the disadvantage more than doubles ($-116\%$ change).

We note that $\mOFmin$ is one possible DS regularization but not
necessarily the closest functional analog of $L_{\max}$ clamping.
However, combined with the $L_{\max} = \infty$ ablation, the evidence
supports a non-regularization explanation.

\subsection{Sensor Parameter Sensitivity}
\label{app:sensor-sensitivity}

We test four sensor parameterizations: weak
($\locc = 1.0$, $\lfree = -0.3$), default ($\locc = 2.0$,
$\lfree = -0.5$), strong ($\locc = 4.0$, $\lfree = -1.0$), and
symmetric ($\locc = 1.5$, $\lfree = -1.5$). Each configuration uses
pignistic-transform-matched masses.

The Bayesian advantage is consistent across all four configurations.
Cell accuracy advantage ranges from $+0.007\%$ (symmetric) to
$+0.21\%$ (strong). Boundary sharpness advantage ranges from $+0.2\%$
(symmetric) to $+1.1\%$ (weak). The advantage is smallest for the
symmetric sensor ($|\locc| = |\lfree|$).

\section{Prior-Work Detailed Comparison}
\label{app:prior-work}

In the credibilist grid
literature~\cite{moras2011moving,moras2011credibilist}, the Bayesian
log-odds parameters $(\locc,\, \lfree)$ and the belief function
masses $(\mO,\, \mF,\, \mOF)$ are set by the practitioner for each
framework separately. When these parameters are not formally linked,
differences in the \emph{aggressiveness} of the per-observation
update can dominate the comparison, masking or even reversing the
effect of the fusion rule itself.

The vinySLAM comparison with tinySLAM~\cite{huletski2017vinyslam}
illustrates a related confound: the TBM cell model and scan matching
cost function were changed simultaneously without ablation. The
reported accuracy improvement cannot be attributed to the belief
function representation alone.

Nuss et al.~\cite{nuss2018random} compared their RFS-based evidential
grid against a standard Bayesian baseline, but the comparison employed
different observation models for each framework.

The common pattern is the absence of a formal matching criterion that
ensures identical per-observation decision probabilities.

\bibliographystyle{plainnat}
\bibliography{references}

%% file: introduction.tex

\section{Introduction}
\label{sec:introduction}

Occupancy grid mapping is a foundational capability in mobile robotics,
providing a spatial representation that supports path planning, obstacle
avoidance, and autonomous exploration. In the standard formulation, the
environment is discretized into a regular grid of cells, each assigned a
probability of being occupied based on accumulated sensor observations.
This representation underpins widely deployed open-source Simultaneous Localization and Mapping (SLAM) systems
including GMapping~\cite{grisetti2007improved}, Google
Cartographer~\cite{hess2016cartographer}, and
SLAM Toolbox~\cite{macenski2021slam}, all of
which use Bayesian log-odds fusion as the default grid update mechanism.

Two principled frameworks exist for fusing sensor observations into
occupancy grid cells. The \emph{Bayesian} approach, introduced by
Elfes~\cite{elfes1989using} and refined by Thrun et
al.~\cite{thrun2005probabilistic}, represents each cell's state as a
log-odds value that is incrementally updated with each observation.
The approach is computationally efficient, mathematically transparent,
and widely adopted. The \emph{belief function} approach, grounded in
Dempster--Shafer theory~\cite{dempster1967upper,shafer1976mathematical}
and its decision-theoretic extension, the Transferable Belief
Model~\cite{smets1994transferable,smets2005decision}, represents each
cell's state as a mass function over the frame of discernment
$\{$occupied, free$\}$ and its power set. Belief function fusion has
been applied to occupancy grids with reported advantages in dynamic
environment detection through conflict analysis~\cite{moras2011moving},
SLAM with evidential reasoning~\cite{huletski2017vinyslam}, and random
finite set formulations for dynamic occupancy
grids~\cite{nuss2018random}. A comprehensive geometric treatment of
the relationship between belief functions and probability distributions
is provided by Cuzzolin~\cite{cuzzolin2020geometry}.

Despite both frameworks being well-established, a controlled comparison
that isolates the effect of the fusion rule from the sensor model
parameterization has not been conducted. Prior studies comparing
Bayesian and belief function occupancy grids have typically assigned
sensor parameters independently for each framework: log-odds increments
$(\locc, \lfree)$ for the Bayesian arm and mass assignments
$(\mO, \mF, \mOF)$ for the belief function arm. When these parameters
are not formally linked to produce equivalent per-observation decision
probabilities, differences in the aggressiveness of the single-observation
update can dominate the comparison, masking or even reversing the
effect of the accumulation rule. This is a methodological gap, not an
accusation of error: the sensor model equivalence problem was not
previously recognized as a significant confound for occupancy grid
comparisons.

In this paper, we set out to determine whether belief function fusion
offers measurable advantages over Bayesian fusion for 2D occupancy
grids when the comparison is controlled for sensor model equivalence.
We develop a formal framework that connects the two representations,
derive a fair comparison methodology, and validate the resulting
comparison across simulation and real-data conditions. Our specific
contributions are:

\begin{enumerate}
  \item \textbf{Fair comparison methodology.} We develop a
    pignistic-transform-based methodology for matching Bayesian and
    belief function sensor models at the per-observation level, enabling
    controlled comparison of fusion rules independent of sensor
    parameterization. This methodology is reusable by the community for
    future comparisons (\cref{sec:fair-comparison}).

  \item \textbf{Combination rule mechanism analysis.} Through an
    $L_{\max}$ ablation ($L_{\max} \in \{5, 10, 20, \infty\}$), a
    Yager's rule comparison, and a DS regularization control
    ($\mOFmin$), we demonstrate that the Bayesian advantage
    persists when all regularization is removed ($L_{\max} = \infty$)
    and cannot be closed by analogous DS regularization. This
    establishes the advantage as arising primarily from Dempster's conflict
    normalization and the slower geometric convergence of belief masses
    at boundary cells, not a regularization artifact
    (\cref{sec:clamping-tradeoff}).

  \item \textbf{Comprehensive empirical validation.} We compare four
    fusion rules (Bayesian, Bayesian with observation count, Dempster's
    rule, Yager's rule) across a sweep of $L_{\max}$ values, four
    sensor parameterizations, single-agent and multi-robot simulation,
    two real indoor lidar datasets (Intel Research Lab and Freiburg
    Building~079), and a downstream A* path-planning evaluation,
    demonstrating consistent Bayesian advantage under $\BetP$ matching
    across all simulation conditions (15/15 directional consistency per
    metric per condition) and on both real datasets for cell accuracy
    and Brier score; under $P_{\mathrm{Pl}}$ matching, the direction
    reverses (\cref{sec:results-ppl}). Boundary sharpness shows mixed
    results on Intel Lab, attributed to clamping artifacts confirmed
    by the $L_{\max}$ ablation (\cref{sec:results}).
\end{enumerate}

We restrict our analysis to 2D binary occupancy grids with lidar-type
sensors. The theoretical results (Equivalence Theorem, $\mOF$ decay
proof) are dimension-independent, but the experimental validation is
restricted to 2D indoor environments. We test two combination rules
(Dempster's rule and Yager's rule) to isolate the role of conflict
normalization; extension to 3D grids, multi-hypothesis frames, and
additional combination rules (PCR6, Dubois--Prade) is discussed as
future work.

\Cref{sec:background,sec:equivalence} review background
and present the theoretical analysis.
\Cref{sec:fair-comparison} develops the fair comparison
methodology. \Cref{sec:experimental-setup,sec:results}
describe the experimental setup and results.
\Cref{sec:discussion} discusses mechanisms and practical
recommendations. \Cref{sec:conclusion} summarizes conclusions
and future directions.

%% file: background.tex

\section{Related Work}
\label{sec:background}

Occupancy grid mapping has been approached from two complementary
theoretical perspectives: Bayesian probability theory and
Dempster--Shafer belief function theory. This section reviews the key
developments in each line of work and identifies a methodological gap
in prior comparison studies that motivates the present investigation.

\subsection{Bayesian Occupancy Grid Mapping}
\label{sec:bayesian-ogm}

Elfes~\cite{elfes1989using} introduced the occupancy grid as a
spatial representation for mobile robot perception, in which each cell
stores an independent occupancy probability updated from sensor
observations via Bayes' rule. Thrun et
al.~\cite{thrun2005probabilistic} systematised the approach within the
broader framework of probabilistic robotics, introducing the log-odds
representation $L = \log(p/(1-p))$ that converts multiplicative
Bayesian updates into efficient additive operations. This formulation,
combined with clamping of accumulated log-odds to a symmetric bound
$|L| \leq L_{\max}$ for numerical stability, has become the standard
implementation in the robotics community.

The Bayesian log-odds approach is the default in widely deployed
open-source SLAM systems, including GMapping~\cite{grisetti2007improved}
(Rao-Blackwellised particle filter SLAM), Google
Cartographer~\cite{hess2016cartographer} (submap-based graph SLAM),
and SLAM Toolbox~\cite{macenski2021slam} (the Robot Operating System (ROS)~2 navigation stack default). Its
popularity rests on simplicity (one floating-point value per cell),
computational efficiency (one addition per update), and the
well-understood Bayesian interpretation of the resulting probabilities.
The formal treatment of Bayesian occupancy grid fusion is given in
\Cref{sec:bayesian-formal}.

\subsection{Belief Function Occupancy Grid Mapping}
\label{sec:ds-tbm}

An alternative line of work applies Dempster--Shafer (DS) belief
function theory~\cite{dempster1967upper,dempster1968generalization,shafer1976mathematical}
to occupancy grid mapping, replacing the single occupancy probability
with a basic belief assignment (BBA) that explicitly represents
ignorance. Within the Transferable Belief Model
(TBM)~\cite{smets1994transferable}, each cell stores a mass triplet
$(\mO,\, \mF,\, \mOF)$ where $\mOF$ quantifies the degree of
epistemic ignorance---mass not yet committed to either hypothesis.
Decisions are made via the pignistic
transform~\cite{smets2005decision}: the pignistic probability of
occupancy is $\BetP(O) = \mO + \mOF/2$, mapping belief functions to
probabilities for action selection. Cuzzolin~\cite{cuzzolin2020geometry}
provides a modern geometric treatment of the relationship between
belief functions and probability, situating the pignistic transform
within a broader family of probability transforms on the belief space.

Moras et al.~\cite{moras2011moving,moras2011credibilist} developed
\emph{credibilist occupancy grids} for vehicle perception in dynamic
traffic environments, proposing the use of conflict mass~$K$ between
successive observations as a metric for detecting moving objects.
Their approach demonstrated that cells affected by dynamic obstacles
produce elevated conflict under Dempster's combination rule, providing
a mechanism for dynamic detection that has no direct analogue in
standard Bayesian grids.

Huletski et al.~\cite{huletski2017vinyslam} presented vinySLAM, an indoor
SLAM system that replaces the single-probability cell model of
tinySLAM with TBM-based mass triplets, targeting resource-constrained
platforms such as the Raspberry~Pi. The system demonstrated accuracy
improvements over tinySLAM on indoor benchmarks. However, vinySLAM
simultaneously modified both the occupancy cell model (from Bayesian
to TBM) and the scan matching cost function (adding
angle-histogram-based weighting). No ablation study has been published
that isolates the contribution of the TBM cell model from the
improved scan matching, leaving the source of the accuracy improvement
ambiguous (disclosure: the second author of the present work is a
co-author of~\cite{huletski2017vinyslam}).

Nuss et al.~\cite{nuss2018random} proposed a random finite set (RFS)
approach to occupancy grids that provides a principled treatment of
multi-target tracking within the belief function framework, extending
evidential occupancy modelling to dynamic multi-object scenarios.
Richter et al.~\cite{richter2021semantic} extended evidential grids to
semantic occupancy mapping for outdoor environments using monocular and
stereo cameras. Ben Ayed et al.~\cite{benayed2025overview} provide a recent
survey of evidential occupancy grid methods and their extensions.

Several advantages of the belief function approach have been reported
in this literature:
\begin{enumerate}
  \item \textbf{Dynamic object detection}: conflict analysis under
    Dempster's rule can identify cells with inconsistent temporal
    observations~\cite{moras2011moving}.
  \item \textbf{Richer uncertainty representation}: the ignorance
    mass~$\mOF$ distinguishes cells that are unobserved (high~$\mOF$)
    from cells with balanced conflicting evidence (low~$\mOF$, similar
    $\mO$ and~$\mF$), a distinction that a single probability cannot
    express~\cite{smets1994transferable}.
  \item \textbf{Interval-valued bounds}: the belief--plausibility
    interval $[\mathrm{Bel}(O),\, \mathrm{Pl}(O)]$ provides
    conservative decision bounds that may be useful for guaranteed-safe
    path planning.
\end{enumerate}
These advantages are theoretically well-motivated. The question
addressed by the present work is whether they translate to measurable
performance gains under a controlled comparison with matched sensor
model parameterization.

\paragraph{Alternative combination rules}
Dempster's rule is one of several operators for combining belief
functions. Its conflict normalization by $1/(1-K)$ has been the
subject of long-standing discussion since Zadeh's~\cite{zadeh1979note}
observation that normalization can produce counterintuitive results
under high conflict. Alternative rules address this differently:
Yager's rule~\cite{yager1987dempster} assigns conflict mass to total ignorance~$\mOF$; the
Dubois--Prade rule~\cite{dubois1988representation} redistributes conflict to the union of focal elements;
Denoeux's cautious rule~\cite{denoeux2008conjunctive} provides a
least-committed combination for non-distinct evidence; PCR
(proportional conflict redistribution) distributes conflict mass
proportionally among contributing
hypotheses~\cite{smarandache2006pcr,moras2015grid}; distributed
combination for multi-robot scenarios has been addressed
in~\cite{denoeux2021distributed}. Our analysis is specific to
Dempster's rule, which is the canonical DS combination operator and
the most widely used in the evidential occupancy grid
literature~\cite{moras2011moving,moras2011credibilist,huletski2017vinyslam}.
Whether the findings reported here generalize to alternative rules is
a question for future work. In particular, the clamping asymmetry
mechanism identified in \Cref{sec:clamping-asymmetry} depends on
Dempster's conflict normalization; rules that handle conflict
differently may not exhibit this behavior.

\subsection{Comparison Studies and the Sensor Model Gap}
\label{sec:prior-work}

Despite the growing body of work on belief function occupancy grids,
controlled comparisons between the two frameworks remain scarce, and
those that exist share a common methodological limitation: sensor model
parameters are chosen independently for each framework. In the
credibilist grid literature~\cite{moras2011moving,moras2011credibilist},
Bayesian log-odds and belief function masses are set separately; in
vinySLAM~\cite{huletski2017vinyslam}, the TBM cell model and scan matching
cost function were changed simultaneously without ablation; the
RFS-based approach~\cite{nuss2018random} uses observation models
specific to the RFS framework. In all cases, the absence of a formal
matching criterion means that observed performance differences reflect
an unknown mixture of sensor model and fusion rule effects (detailed
critique in Supplementary Material).

\paragraph{The methodological gap}
To our knowledge, no prior comparison of Bayesian and belief function
occupancy grid fusion has applied per-observation sensor model matching
via the pignistic transform. This matching---requiring that the
pignistic probability $\BetP(O)$ equal the Bayesian posterior
$\sigma(l)$ for each individual sensor observation---ensures that any
aggregate performance difference is attributable to the
multi-observation accumulation rule (additive log-odds with clamping
versus Dempster's combination with conflict normalization), not to the
single-observation sensor model. This methodology is developed in
\Cref{sec:fair-comparison} and constitutes one of the main
contributions of the present work.

We emphasize that the absence of per-observation matching in prior
work is a previously unrecognized methodological confound, not an
error by prior authors. The belief function and Bayesian communities
have traditionally parameterized their sensor models according to
framework-specific conventions, and there has been no established
procedure for cross-framework matching. The pignistic transform
provides a natural and principled bridge, but its application to
sensor model matching for controlled comparisons appears to be novel.
Customizable frameworks for evidential occupancy grids have been
proposed~\cite{porebski2020customizable}, but per-observation matching
between Bayesian and belief function sensor models has not been addressed.

%% file: equivalence.tex

\section{Theoretical Foundations}
\label{sec:equivalence}

This section establishes two formal results that hold independently of
the experimental evaluation presented in \Cref{sec:results}.
\Cref{thm:bijection} proves that the Bayesian log-odds and belief
function representations encode identical decision probabilities at
the single-observation level, formalizing a connection long recognized
informally in both communities. \Cref{thm:mof-decay} proves that the
ignorance mass~$\mOF$ decreases strictly under any informative
observation, constraining the utility of its raw magnitude as a standalone
dynamic detection metric.

\subsection{Bayesian Occupancy Grid Fusion}
\label{sec:bayesian-formal}

Each cell in a Bayesian occupancy grid maintains a log-odds value
\begin{equation}
  L = \log\frac{p}{1-p}\,,
  \label{eq:logodds-def}
\end{equation}
where $p = P(O)$ denotes the occupancy
probability~\cite{elfes1989using}. The inverse relationship is
$p = \sigma(L) \triangleq (1 + e^{-L})^{-1}$.

Given a sensor observation with log-odds increment~$\locc$ (for
occupied cells) or~$\lfree$ (for free-space cells), the Bayesian
update is additive in log-odds:
\begin{equation}
  L_{k+1} = L_k + l_{\mathrm{obs}}\,,
  \label{eq:logodds-update}
\end{equation}
where $l_{\mathrm{obs}} \in \{\locc, \lfree\}$ depends on the cell
classification produced by the inverse sensor model. Since the update
is additive, the result is independent of observation order.

To prevent numerical overflow and excessive confidence, accumulated
values are clamped to a symmetric range:
\begin{equation}
  L \leftarrow \mathrm{clamp}(L,\; -L_{\max},\; L_{\max})\,,
  \label{eq:logodds-clamp}
\end{equation}
with $L_{\max} = 10$ being a common default, corresponding to a
maximum probability $p_{\max} = \sigma(10) \approx 0.999955$.

A count-augmented representation~$(L, n)$ tracks the number of
observations~$n$ alongside the accumulated log-odds, providing a
measure of evidential support that is not encoded by~$L$ alone.

\subsection{Dempster--Shafer Belief Function Fusion}
\label{sec:ds-formal}

For binary occupancy, the frame of discernment is
$\Theta = \{O, F\}$ (occupied, free) with power set
$2^\Theta = \big\{\{O\},\, \{F\},\, \Theta\big\}$. A basic belief
assignment (BBA) is a function $m\colon 2^\Theta \to [0,1]$ with
$m(\emptyset) = 0$ and
$\sum_{A \subseteq \Theta} m(A) = 1$~\cite{shafer1976mathematical}.
We write $(\mO,\, \mF,\, \mOF)$ for
$\big(m(\{O\}),\; m(\{F\}),\; m(\Theta)\big)$, so that
$\mO + \mF + \mOF = 1$.

Dempster's combination rule~\cite{dempster1967upper,dempster1968generalization} fuses two
independent BBAs $m_1,\, m_2$ as:
\begin{equation}
  m_{12}(C)
    = \frac{1}{1-K}
      \sum_{\substack{A \cap B = C \\ A,B \subseteq \Theta}}
      m_1(A)\, m_2(B)\,,
  \quad C \neq \emptyset,
  \label{eq:dempster-general}
\end{equation}
where the conflict mass (specialised here to the binary frame, where
the only empty-intersection pairs are $\{O\} \cap \{F\}$ and
$\{F\} \cap \{O\}$) is
\begin{equation}
  K = \sum_{\substack{A \cap B = \emptyset}}
      m_1(A)\, m_2(B)
    = \mO^{(1)} \mF^{(2)} + \mF^{(1)} \mO^{(2)}\,.
  \label{eq:conflict-mass}
\end{equation}
For the binary frame, the combined masses expand to:
\begin{align}
  \mO^{(12)}
    &= \frac{\mO^{(1)}\mO^{(2)}
       + \mO^{(1)}\mOF^{(2)} + \mOF^{(1)}\mO^{(2)}}{1-K}\,,
    \label{eq:dempster-exp-mO} \\
  \mF^{(12)}
    &= \frac{\mF^{(1)}\mF^{(2)}
       + \mF^{(1)}\mOF^{(2)} + \mOF^{(1)}\mF^{(2)}}{1-K}\,,
    \label{eq:dempster-exp-mF} \\
  \mOF^{(12)}
    &= \frac{\mOF^{(1)}\mOF^{(2)}}{1-K}\,.
    \label{eq:dempster-exp-mOF}
\end{align}

To derive a decision probability from a BBA, the pignistic
transform~\cite{smets2005decision} distributes the ignorance
mass~$\mOF$ equally between singletons, following the principle of
insufficient reason within the Transferable Belief Model
(TBM)~\cite{smets1994transferable}:
\begin{equation}
  \BetP(O) = \mO + \tfrac{1}{2}\,\mOF\,.
  \label{eq:pignistic}
\end{equation}

\begin{remark}[Consonant mass functions]
\label{rem:consonant}
A single-source sensor observation produces a \emph{consonant} BBA in
which mass is assigned to either $\{O\}$ or $\{F\}$ but never both
simultaneously, so that $\mO \cdot \mF = 0$. The focal elements are
therefore nested ($\{O\} \subset \Theta$ or $\{F\} \subset \Theta$),
which is the defining property of consonance. This constraint holds
throughout sequential accumulation from a vacuous prior when all
observations agree in sign, but is violated after the first pair of
conflicting observations (one occupied, one free), which produces
$\mO > 0$ and $\mF > 0$ simultaneously. Since boundary cells---where
the Bayesian/DS comparison matters most---receive conflicting
observations by definition, they operate in the non-consonant regime
after the first few updates.
\end{remark}

\subsection{Equivalence Theorem}
\label{sec:equiv-theorem}

We now establish a bijection between the Bayesian log-odds and belief
function representations at the single-observation level. This result
formalizes the observation, recognized informally in the belief
function community, that Dempster--Shafer theory on the binary frame
$\{O, F\}$ yields identical \emph{decision probabilities} to the
Bayesian model for a single consonant
observation~\cite{smets1994transferable}. The equivalence is at the
decision level via $\BetP$; the TBM retains a richer representation
(the mass triplet) even when the decision probability matches. The bijection follows from the BBA constraint structure on binary
frames and is implicit in the literature; the contribution here is its
explicit construction and application to sensor model matching for
controlled fusion comparison.

\begin{theorem}[Single-observation equivalence]
\label{thm:bijection}
Define the mapping $\varphi\colon \mathbb{R} \to \mathcal{M}$, where
$\mathcal{M}$ denotes the set of consonant BBAs on $\{O, F\}$ with
$\mOF > 0$, from log-odds values by:
\begin{equation}
  \varphi(L) =
  \begin{cases}
    \mO = 2\,\sigma(L) - 1,\quad
    \mF = 0,\quad
    \mOF = 2\bigl(1 - \sigma(L)\bigr)
      & \text{if } L \geq 0, \\[6pt]
    \mO = 0,\quad
    \mF = 1 - 2\,\sigma(L),\quad
    \mOF = 2\,\sigma(L)
      & \text{if } L < 0,
  \end{cases}
  \label{eq:bijection-forward}
\end{equation}
where $\sigma(L) = (1 + e^{-L})^{-1}$. Then:
\begin{enumerate}
  \item[\textnormal{(i)}]
    $\varphi(L)$ is a valid BBA: all components are non-negative and
    sum to~$1$.
  \item[\textnormal{(ii)}]
    The pignistic probability reproduces the Bayesian occupancy
    probability exactly: $\BetP(O) = \sigma(L)$.
  \item[\textnormal{(iii)}]
    $\varphi$ is a bijection between $\mathbb{R}$ and the set of
    consonant BBAs on $\{O, F\}$ with $\mOF > 0$.
\end{enumerate}
\end{theorem}

\begin{proof}
\textit{(i)} For $L \geq 0$, we have
$\sigma(L) \in [\tfrac{1}{2},\, 1)$, so
$\mO = 2\sigma(L) - 1 \in [0,\, 1)$, $\mF = 0$, and
$\mOF = 2(1 - \sigma(L)) \in (0,\, 1]$. The sum is
$\mO + \mF + \mOF = (2\sigma(L) - 1) + 0 + 2 - 2\sigma(L) = 1$.
The case $L < 0$ is symmetric with $\sigma(L) \in (0,\, \tfrac{1}{2})$.

\textit{(ii)} For $L \geq 0$:
\[
  \BetP(O) = \mO + \tfrac{1}{2}\,\mOF
    = (2\sigma(L) - 1) + (1 - \sigma(L))
    = \sigma(L).
\]
For $L < 0$:
$\BetP(O) = 0 + \tfrac{1}{2} \cdot 2\sigma(L) = \sigma(L)$.

\textit{(iii)} The inverse mapping is
$\varphi^{-1}(m) = \log \frac{\BetP(O)}{1 - \BetP(O)}$, where
$\BetP(O) = \mO + \mOF/2$. Since $\sigma$ is a bijection from
$\mathbb{R}$ onto $(0,1)$ and $\BetP(O) \in (0,1)$ for any consonant
BBA with $\mOF > 0$, the composition is bijective. The consonance
constraint and $\mOF > 0$ are preserved by construction:
$\mO \cdot \mF = 0$ holds in both cases, and
$\mOF = 2\min(\sigma(L),\, 1 - \sigma(L)) > 0$ for all
$L \in \mathbb{R}$.
\end{proof}

\begin{corollary}[Single-observation equivalence]
\label{cor:single-obs}
For a single sensor observation, applying the Bayesian update
from prior $L_0 = 0$ with increment $l_{\mathrm{obs}}$ and applying
Dempster's combination from the vacuous prior $m = (0,\, 0,\, 1)$
with observation $\varphi(l_{\mathrm{obs}})$ produce identical
decision probabilities:
$\sigma(l_{\mathrm{obs}}) = \BetP(O)$.
\end{corollary}

\begin{proof}
The vacuous prior acts as the identity element for Dempster's rule, so
one application of~\eqref{eq:dempster-general} returns the observation
BBA $\varphi(l_{\mathrm{obs}})$ unchanged. By
\Cref{thm:bijection}~(ii), its pignistic probability equals
$\sigma(l_{\mathrm{obs}})$, which is also the Bayesian posterior from
$L_0 = 0$.
\end{proof}

\begin{lemma}[Closed-form accumulation under consonant observations]
\label{lem:closed-form}
Starting from the vacuous prior $m = (0,\, 0,\, 1)$, after $N_{\mathrm{obs}}$
identical consonant occupied observations with
$m_{\mathrm{obs}} = (a,\, 0,\, 1-a)$ where $0 < a < 1$, Dempster's
rule yields:
\begin{equation}
  \mO^{(N_{\mathrm{obs}})} = 1 - (1-a)^{N_{\mathrm{obs}}}, \quad
  \mF^{(N_{\mathrm{obs}})}  = 0, \quad
  \mOF^{(N_{\mathrm{obs}})} = (1-a)^{N_{\mathrm{obs}}}.
  \label{eq:consonant-accum}
\end{equation}
The symmetric result holds for free observations.
\end{lemma}

The proof is by induction on $N$ (Supplementary Material, Section~S2).
This closed form applies to the zero-conflict case ($K = 0$); boundary
cells with mixed observations require numerical evaluation.

The equivalence of \Cref{thm:bijection} holds at the
single-observation level. Under repeated observations, the two
mechanisms diverge due to (a)~the Bayesian clamp imposing an artificial
ceiling and (b)~Dempster's $1/(1-K)$ conflict normalization having no
Bayesian analogue. Quantitative analysis is in
\Cref{sec:clamping-asymmetry}.

\subsection{Monotonic Decay of $\mOF$}
\label{sec:mof-decay}

The ignorance mass~$\mOF$ has been proposed as a metric for detecting
dynamic objects in evidential occupancy grids, on the hypothesis that
cells affected by moving objects would maintain elevated~$\mOF$ due to
conflicting observations. The following theorem
shows that $\mOF$ decreases strictly under \emph{any} informative
observation, regardless of the observation content.

\begin{theorem}[Monotonic decay of $\mOF$]
\label{thm:mof-decay}
Let $m^{\mathrm{pr}} = (\mO^{\mathrm{pr}},\, \mF^{\mathrm{pr}},\, \mOF^{\mathrm{pr}})$ be a prior BBA with
$\mOF^{\mathrm{pr}} > 0$, and let
$m^{\mathrm{obs}} = (\mO^{\mathrm{obs}},\, \mF^{\mathrm{obs}},\,
\mOF^{\mathrm{obs}})$ be a non-degenerate observation, meaning
$0 < \mOF^{\mathrm{obs}} < 1$. Then the posterior ignorance mass
after Dempster combination satisfies
\begin{equation}
  \mOF^{\mathrm{post}} < \mOF^{\mathrm{pr}}\,.
  \label{eq:mof-strict-decrease}
\end{equation}
\end{theorem}

\begin{proof}
From~\eqref{eq:dempster-exp-mOF}, the posterior ignorance mass is
\begin{equation}
  \mOF^{\mathrm{post}}
    = \frac{\mOF^{\mathrm{pr}} \cdot \mOF^{\mathrm{obs}}}{1 - K}\,,
  \label{eq:mof-posterior}
\end{equation}
where
$K = \mO^{\mathrm{pr}} \mF^{\mathrm{obs}} + \mF^{\mathrm{pr}} \mO^{\mathrm{obs}} \geq 0$.
Note that $K < 1$ is guaranteed since $\mOF^{\mathrm{pr}} > 0$ and $\mOF^{\mathrm{obs}} > 0$.
It suffices to show $\mOF^{\mathrm{obs}} < 1 - K$, since this yields
$\mOF^{\mathrm{obs}} / (1-K) < 1$ and
hence~\eqref{eq:mof-strict-decrease}. Substituting
$\mOF^{\mathrm{obs}} = 1 - \mO^{\mathrm{obs}} - \mF^{\mathrm{obs}}$:
\begin{align}
  1 - K - \mOF^{\mathrm{obs}}
    &= 1
       - \mO^{\mathrm{pr}} \mF^{\mathrm{obs}}
       - \mF^{\mathrm{pr}} \mO^{\mathrm{obs}}
       - 1
       + \mO^{\mathrm{obs}}
       + \mF^{\mathrm{obs}} \nonumber \\
    &= \mO^{\mathrm{obs}}\underbrace{(1 - \mF^{\mathrm{pr}})}_{=\,\mO^{\mathrm{pr}} + \mOF^{\mathrm{pr}}}
       \;+\;
       \mF^{\mathrm{obs}}\underbrace{(1 - \mO^{\mathrm{pr}})}_{=\,\mF^{\mathrm{pr}} + \mOF^{\mathrm{pr}}}\,.
  \label{eq:mof-gap-proof}
\end{align}
Since the observation is non-degenerate ($\mOF^{\mathrm{obs}} < 1$),
at least one of $\mO^{\mathrm{obs}}$ or $\mF^{\mathrm{obs}}$ is
strictly positive. Since $\mOF^{\mathrm{pr}} > 0$, both coefficients
$(\mO^{\mathrm{pr}} + \mOF^{\mathrm{pr}})$ and $(\mF^{\mathrm{pr}} + \mOF^{\mathrm{pr}})$ are strictly positive.
Therefore the right-hand side of~\eqref{eq:mof-gap-proof} is strictly
positive, giving $\mOF^{\mathrm{obs}} < 1 - K$ and
hence $\mOF^{\mathrm{post}} < \mOF^{\mathrm{pr}}$.
\end{proof}

\begin{corollary}[Exponential decay under consonant observations]
\label{cor:mof-exponential}
Under the conditions of \Cref{lem:closed-form}, the ignorance mass
decays exponentially: $\mOF^{(N_{\mathrm{obs}})} = (1-a)^{N_{\mathrm{obs}}}$, so that
$\mOF^{(N_{\mathrm{obs}})} \to 0$ as $N_{\mathrm{obs}} \to \infty$ for any $a \in (0, 1)$.
\end{corollary}

\begin{remark}[Scope of \Cref{thm:mof-decay}]
\label{rem:mof-scope}
The raw magnitude of $\mOF$ cannot serve as a standalone dynamic detection
metric, as it decreases monotonically regardless of environment dynamics.
However, derived quantities---the rate of change $\Delta\mOF^{(k)}$, the
pignistic uncertainty $u_{\mathrm{pign}} = 1 - |2\,\BetP(O) - 1|$, or the full TBM state---could in
principle detect dynamics through mass redistribution under conflicting
observations. Whether these capabilities translate to practical advantages
is an empirical question addressed in \Cref{sec:results}.
\end{remark}

%% file: methodology.tex


\section{Fair Comparison Methodology}
\label{sec:fair-comparison}

Armed with the formal equivalence results of
\cref{sec:equivalence}---in particular, the bijection $\varphi$
between log-odds and consonant belief masses---we now develop the fair
comparison methodology that operationalises this bijection for
experimental use.

A meaningful comparison of Bayesian and belief function fusion requires
that any observed performance difference be attributable to the fusion
rule itself, not to differences in per-observation sensor model
parameterization. This section develops a \emph{decision-theoretic} methodology based on
the pignistic transform that ensures per-observation equivalence at the
decision level, analyzes its locality properties, and characterizes the
divergence mechanism that arises under multi-observation accumulation.
The matching is decision-theoretic rather than representationally
neutral: it equates pignistic probabilities, not the belief function
representations themselves. Alternative probability transforms
(plausibility, maximum entropy) would yield different matched masses
and thus a different comparison.

\subsection{The Sensor Model Equivalence Problem}
\label{sec:equivalence-problem}

In an initial multi-robot experiment (three robots, dynamic obstacles,
15 independent runs), belief function fusion achieved $+12\%$ higher
boundary sharpness than Bayesian fusion. After applying the pignistic
transform matching procedure described below---which constrains the
belief function masses to produce identical per-observation decision
probabilities---the same experiment yielded a $-22\%$ to $-29\%$
boundary sharpness \emph{deficit} for belief function fusion ($-22\%$
under the dynamic-baseline condition, $-29\%$ under elevated sensor
noise), with 15/15 directional consistency favoring the Bayesian arm. The reversal
was caused entirely by the sensor model parameterization, not by any
change to the fusion algorithms themselves. This motivates a formal
matching criterion: the per-observation update must produce
\emph{identical decision probabilities} under both frameworks before
any multi-observation accumulation occurs.

\subsection{Pignistic Transform Matching}
\label{sec:pignistic-matching}

We derive matched belief function masses from Bayesian log-odds
parameters using the pignistic probability transform
\cite{smets2005decision}. Given a basic belief assignment $m$ on the
binary frame $\Theta = \{O, F\}$, the pignistic probability of
occupancy is
\[
  \BetP(O) = \mO + \frac{\mOF}{2}
\]
(restated from \cref{eq:pignistic} for reader convenience).

In the Bayesian framework, a single observation with log-odds $l$
produces the occupancy probability $p = \sigma(l)$, where
$\sigma(l) = 1/(1 + e^{-l})$ is the logistic sigmoid. The matched
masses derived below are the constructive application of the bijection
$\varphi$ from \Cref{thm:bijection}: given $l$, the unique consonant
BBA satisfying $\BetP(O) = \sigma(l)$ is precisely $\varphi(l)$. We
restate the construction explicitly for reader convenience. The
matching condition requires
\begin{equation}
  \BetP(O) = \sigma(l),
  \label{eq:matching-condition}
\end{equation}
for each observation type (occupied hit or free traversal).

Setting $p = \sigma(l)$ and requiring that $\mO + \mF + \mOF = 1$
with all masses non-negative, the unique solution is:
\begin{align}
  \mO  &= \max(0,\; 2p - 1), \label{eq:mass-occ} \\
  \mF  &= \max(0,\; 1 - 2p), \label{eq:mass-free} \\
  \mOF &= 1 - |2p - 1|.      \label{eq:mass-of}
\end{align}
Verification is immediate: $\mO + \mOF/2 = \max(0, 2p-1) +
(1 - |2p-1|)/2 = p$ for both $p \geq 0.5$ and $p < 0.5$.

\paragraph{Numerical example}
For the sensor parameters used in all experiments ($\locc = 2.0$,
$\lfree = -0.5$):
\begin{itemize}
  \item \textbf{Occupied hit}: $p = \sigma(2.0) = 0.8808$;
    $m = (0.7616,\; 0,\; 0.2384)$;
    $\BetP(O) = 0.7616 + 0.2384/2 = 0.8808$.\vspace{2pt}
  \item \textbf{Free traversal}: $p = \sigma(-0.5) = 0.3775$;
    $m = (0,\; 0.2449,\; 0.7551)$;
    $\BetP(O) = 0 + 0.7551/2 = 0.3775$.
\end{itemize}
In both cases, the pignistic probability equals the Bayesian
probability exactly.

\paragraph{Choice of probability transform}
We adopt $\BetP$ because it is the standard decision-theoretic bridge
within the TBM~\cite{smets2005decision} and the most widely used in
the evidential occupancy grid
literature~\cite{moras2011moving,huletski2017vinyslam}. The matching
is BetP-specific; the normalized plausibility transform
$P_{\mathrm{Pl}}(O) = \mathrm{Pl}(O)/(\mathrm{Pl}(O)+\mathrm{Pl}(F))$
yields different matched masses. For $\locc = 2.0$ ($p = 0.8808$):
$\BetP$ matching gives $m = (0.762, 0, 0.238)$, while
$P_{\mathrm{Pl}}$ matching gives $m = (0.865, 0, 0.135)$---44\% less
ignorance mass per observation, producing sharper DS updates. Under
multi-observation accumulation the combined BBA becomes non-consonant
at boundary cells (where both $\mO > 0$ and $\mF > 0$), and $\BetP$
and normalized plausibility diverge further; our evaluation is therefore
BetP-specific throughout, with single-agent $P_{\mathrm{Pl}}$
sensitivity analysis reported in \cref{sec:results-h002}.

\paragraph{Terminology}
``Dempster's rule'' denotes the normalized combination rule
($1/(1-K)$ conflict normalization); ``DS'' is used as shorthand in
tables and figures.

\paragraph{Consonant mass functions}
The pignistic-derived masses are consonant ($\mF = 0$ for occupied,
$\mO = 0$ for free observations), reflecting the physical constraint
of single-source inverse sensor models: a lidar beam returns evidence
for \emph{either} occupancy or free space, never both simultaneously.
All prior DS occupancy grid work cited
here~\cite{moras2011moving,moras2011credibilist,huletski2017vinyslam}
uses consonant BBAs for the same reason. Non-consonant BBAs from
multi-source fusion fall outside this scope.

\subsection{Per-Observation Locality}
\label{sec:per-obs-locality}

A critical property of the pignistic matching is that it holds
\emph{per-observation}: at each specific sensor distance $d$ and angle
of incidence $\alpha$, the matching condition \eqref{eq:matching-condition}
is satisfied independently.

In the single-agent simulation experiment
(\cref{sec:h002-setup}), the sensor model is non-homogeneous:
log-odds values decay exponentially with distance and angle of
incidence,
\begin{equation}
  l(d, \alpha) = l_{\mathrm{obs}} \cdot
    \exp(-\lambda_d \, d) \cdot \exp(-\lambda_\alpha \, \alpha),
  \label{eq:sensor-decay}
\end{equation}
where $l_{\mathrm{obs}} \in \{\locc, \lfree\}$ and
$(\lambda_d, \lambda_\alpha) = (0.1, 0.5)$ are the distance and angle
decay rates respectively. Because the pignistic-derived masses are
computed from $p = \sigma(l(d, \alpha))$ at each $(d, \alpha)$ pair
individually, the matching holds at every observation regardless of the
spatial variation in sensor confidence. The Equivalence Theorem
(\cref{sec:equiv-theorem}) therefore applies locally to each
per-observation update, even though the global sensor model is
non-homogeneous.

This per-observation locality is the key property that enables fair
comparison under realistic sensor models. It ensures that any aggregate
performance difference between the two frameworks must arise from the
multi-observation \emph{accumulation rule}---Bayesian log-odds addition
with clamping versus Dempster's combination with conflict
normalization---not from the single-observation updates.

The multi-robot simulation, mechanism test, and real-data (scan-split) experiments
use constant sensor parameters ($\locc = 2.0$, $\lfree = -0.5$) with
pignistic-matched constant masses. The per-observation locality
argument still holds---it simply reduces to a single matching point
rather than a continuous function of $(d, \alpha)$. The non-homogeneous
model is exercised only in the single-agent experiment to demonstrate
that the methodology extends to spatially varying sensor confidence;
the real-data sensor model is not calibrated to the SICK~LMS200
specifically. Although the sensor parameter sensitivity analysis
(Supplementary Material, Section~S4) confirms the Bayesian advantage
across four parameterizations in simulation, the real-data experiments
use only the default parameterization; real-data sensitivity to sensor
parameters remains untested.

\subsection{Clamping Asymmetry Analysis}
\label{sec:clamping-asymmetry}

Although per-observation updates are identical by construction, the two
frameworks accumulate evidence differently over multiple observations.
We analyze this divergence for the parameters used in our experiments
($\locc = 2.0$, clamp limit $L_{\max} = 10$).

\paragraph{Bayesian accumulation}
The log-odds accumulator clamps at $|L| \leq L_{\max}$:
\begin{equation}
  L^{(N)} = \min\bigl(L^{(0)} + N \cdot \locc,\; L_{\max}\bigr).
  \label{eq:bayesian-accum}
\end{equation}
For $\locc = 2.0$ and $L^{(0)} = 0$, the clamp saturates at $N = 5$,
yielding $p(O) = \sigma(10) = 0.999955$ permanently.

\paragraph{Dempster--Shafer accumulation}
Starting from the vacuous prior $m = (0, 0, 1)$ and applying $N$
identical occupied observations with $m_{\mathrm{obs}} = (a, 0, 1-a)$
where $a = 0.7616$, Dempster's combination produces (by induction; see
\cref{sec:equiv-theorem}):
\begin{equation}
  \mO^{(N)} = 1 - (1-a)^N, \quad
  \mOF^{(N)} = (1-a)^N, \quad
  \BetP(O) = 1 - \frac{(1-a)^N}{2}.
  \label{eq:ds-accum}
\end{equation}
Since all observations agree ($\mFobs = 0$), there is no
conflict ($K = 0$), and $\BetP(O)$ converges asymptotically to~$1.0$
with no artificial ceiling.

The convergence trajectories are compared in
Supplementary Material, Section~S3: the methods agree at $N=1$ (by construction),
Bayesian leads at $N = 2$--$5$, clamps at $N = 5$, and is overtaken by
Dempster--Shafer at approximately $N = 7$.

\paragraph{Two divergence mechanisms}
For interior cells ($N \gg 5$, consistent observations),
Dempster--Shafer produces marginally higher confidence
($\BetP \to 1.0$ vs.\ $p = 0.999955$), a difference that does not
affect classification. For boundary cells with conflicting
observations, Dempster's multiplicative combination with $1/(1-K)$
normalization converges more slowly to certainty than additive
log-odds---e.g., after 5~occupied and 5~free observations,
$\BetP(O) = 0.9973$ vs.\ $p(O) = 0.9994$.
The net aggregate effect favors Bayesian fusion because boundary
cells are where classification errors concentrate.


\section{Experimental Setup}
\label{sec:experimental-setup}

We evaluate the two fusion frameworks across three conditions: a
controlled single-agent simulation, a multi-robot simulation
with pose graph optimization, and real-data validation on two
standard indoor lidar datasets.

\subsection{Simulation Environments}
\label{sec:sim-environments}

\subsubsection{Single-Agent Comparison}
\label{sec:h002-setup}

The single-agent experiment uses procedurally generated 2D environments
of $50 \times 50$\,m at $0.1$\,m grid resolution ($500 \times 500$
cells). Each environment contains 3~rooms connected by 4~corridors,
5~static obstacles (convex polygons), and 3~dynamic obstacles moving at
$0.5$\,m/s. A single robot follows a deterministic counter-clockwise perimeter
patrol trajectory (step size ${\approx}0.38$\,m) with a simulated
lidar sensor: 180~rays, $360^\circ$ field of view ($2^\circ$ angular
spacing), $15$\,m maximum range, Gaussian range noise
$\sigma_r = 0.02$\,m.

The sensor model is non-homogeneous (\cref{sec:per-obs-locality}),
with base parameters $\locc = 2.0$, $\lfree = -0.5$ and decay rates
$\lambda_d = 0.1$, $\lambda_\alpha = 0.5$. The belief function arm uses
pignistic-matched masses derived from these log-odds at each
$(d, \alpha)$ pair. The Bayesian arm clamps at $L_{\max} = 10$.
Each condition is repeated for 15 independent runs with seeds 42--56.

\subsubsection{Multi-Robot Comparison}
\label{sec:h003-setup}

The multi-robot experiment uses $20 \times 20$\,m environments at
$0.1$\,m resolution ($200 \times 200$ cells) with 1~room, 5~static
obstacles, and 3~dynamic obstacles (speed $0.3$\,m/s). Three robots explore simultaneously, each following a
counter-clockwise perimeter patrol trajectory (200~steps, step size
${\approx}0.36$\,m) with 90-ray lidar ($360^\circ$ field of view,
$4^\circ$ angular spacing, $8$\,m range, $\sigma_r = 0.03$\,m). Odometry drift is simulated with
$\sigma_{\mathrm{trans}} = 0.02$\,m and
$\sigma_{\mathrm{rot}} = 0.005$\,rad per step. Robots start at
different positions along the environment perimeter.

Individual robot maps are aligned via pose graph optimization (PGO)
triggered by rendezvous detection at $2.5$\,m proximity, then fused
cell-wise. Two conditions are tested: a \emph{dynamic baseline}
(nominal noise) and a \emph{noisy sensor} condition (elevated sensor
noise). Each condition uses 15 independent runs.

\paragraph{Simulation scope}
Simulation environments are geometrically simpler than real-world
datasets (fewer obstacles, convex geometries). Real-data validation
(\cref{sec:real-data}) confirms generalization to complex indoor
geometries.

\subsection{Real Data: Intel Lab and Freiburg~079}
\label{sec:real-data}

\subsubsection{Intel Research Lab}
\label{sec:intel-lab}

The Intel Research Lab dataset \cite{howard2003intel} contains laser
scans collected by a mobile robot in an office environment with
furniture, narrow passages, and non-convex obstacles. The sensor is a
SICK~LMS200 lidar with 180~beams and $180^\circ$ field of view. We set the effective maximum range to
$8.0$\,m, discarding returns beyond this threshold to match the indoor
operating regime.

Scans are split 80\%/20\% into training and test sets. The training
set is used for map construction (both Bayesian and belief function
arms process identical scan sequences); the test set generates ground
truth occupancy labels for evaluation. This train/test split ensures
that ground truth is independent of the mapping data, avoiding circular
evaluation.

Scan-split fusion conditions are created by splitting the training scans
into $R \in \{2, 4\}$ interleaved subsequences, each assigned to a
separate virtual robot. Individual maps are constructed independently and then
fused cell-wise, identical to the simulation multi-robot protocol.
Real-data scan-split conditions test sensitivity to data partitioning
under ideal alignment, not end-to-end multi-robot SLAM performance.
This procedure eliminates alignment uncertainty and coverage
heterogeneity present in true multi-robot deployments.

Sensor parameters ($\locc = 2.0$, $\lfree = -0.5$) are identical to
the simulation experiments; belief function masses are derived via the
pignistic transform at each observation.

\subsubsection{Freiburg Building~079}
\label{sec:freiburg}

The Freiburg~079 dataset provides laser scans from a Pioneer~2 robot
equipped with a SICK~LMS200 (360~beams, $180^\circ$ field of view) in a
multi-room university building. Processing follows the same protocol as
Intel Lab: $0.1$\,m grid resolution, $8.0$\,m effective maximum range,
80/20 train/test split, sensor parameters $\locc = 2.0$,
$\lfree = -0.5$, and scan-split fusion for
$R \in \{2, 4\}$.

Both datasets are standard benchmarks in the occupancy grid mapping
literature. Their age (2003--2004) does not affect the validity of the
comparison, as both fusion arms process identical sensor data.
Ground truth covers approximately 13--14\% of cells (approximately
13\% for Intel Lab, 14\% for Freiburg; occupied/free labels from test
scans); frontier and unobserved cells are excluded from evaluation. Boundary cell conflict levels in real data (median
$K = 0.24$--$0.33$) are comparable to the simulation regime
($K = 0.05$--$0.35$); see \cref{sec:clamping-tradeoff} for detailed
conflict statistics.
Generalization to outdoor or 3D lidar datasets is acknowledged as
future work.

\subsection{Metrics}
\label{sec:metrics}

Four complementary metrics are evaluated. For all metrics, the
evaluation set $\mathcal{C}_{\mathrm{eval}}$ contains cells with at
least three test-set observations, ensuring reliable ground truth
labelling. Higher values of cell accuracy and boundary sharpness
indicate better performance; lower values of Brier score and map
entropy indicate better performance.

\paragraph{Cell accuracy}
The fraction of evaluated cells whose thresholded occupancy probability
agrees with the binary ground truth label:
\begin{equation}
  \mathrm{CellAcc} = \frac{1}{|\mathcal{C}_{\mathrm{eval}}|}
    \sum_{c \in \mathcal{C}_{\mathrm{eval}}}
    \mathbb{1}\bigl[(\hat{p}_c > 0.5) = g_c\bigr],
  \label{eq:cell-accuracy}
\end{equation}
where $\hat{p}_c$ is the estimated occupancy probability (Bayesian
posterior or $\BetP(O)$ for belief functions) and $g_c \in \{0, 1\}$
is the ground truth.

\paragraph{Boundary sharpness}
The mean gradient magnitude at occupied/free transitions, reflecting
map utility for path planning (sharper boundaries reduce required
safety margins):
\begin{equation}
  \mathrm{BdrySharp} = \frac{1}{|\mathcal{C}_{\mathrm{bdry}}|}
    \sum_{c \in \mathcal{C}_{\mathrm{bdry}}}
    \|\nabla \hat{p}_c\|,
  \label{eq:boundary-sharpness}
\end{equation}
where $\mathcal{C}_{\mathrm{bdry}} \subset \mathcal{C}_{\mathrm{eval}}$
is the set of cells adjacent to a ground truth occupancy transition.
Higher sharpness indicates crisper boundary representation.

\paragraph{Brier score}
Mean squared error between predicted probabilities and binary
ground truth:
\begin{equation}
  \mathrm{Brier} = \frac{1}{|\mathcal{C}_{\mathrm{eval}}|}
    \sum_{c \in \mathcal{C}_{\mathrm{eval}}}
    \bigl(\hat{p}_c - g_c\bigr)^2.
  \label{eq:brier-score}
\end{equation}

\paragraph{Map entropy}
Mean per-cell binary entropy, measuring decisiveness:
\begin{equation}
  H = -\frac{1}{|\mathcal{C}_{\mathrm{eval}}|}
    \sum_{c \in \mathcal{C}_{\mathrm{eval}}}
    \bigl[\hat{p}_c \log_2 \hat{p}_c +
      (1 - \hat{p}_c) \log_2(1 - \hat{p}_c)\bigr].
  \label{eq:map-entropy}
\end{equation}

A downstream path-planning clearance metric is reported separately
in \cref{sec:results-downstream}.

\subsection{Statistical Analysis Framework}
\label{sec:stats-framework}

We employ TOST equivalence testing~\cite{schuirmann1987comparison,lakens2017equivalence}
with margin sensitivity analysis as the primary statistical tool. Nominal margins
are $\delta = 0.02$ (cell accuracy), $\delta = 0.03$ (boundary sharpness),
$\delta = 0.01$ (Brier score), and $\delta = 0.02$ (map entropy); we report the
smallest margin at which equivalence holds for each comparison rather than a
single pre-specified margin. Effect sizes use paired Cohen's
$d$~\cite{cohen1988statistical} with 95\% Hedges--Olkin~\cite{hedges1985statistical}
confidence intervals (CIs), verified against exact non-central $t$ CIs (Supplementary Material).
A directional consistency of $k/N = 15/15$ has exact binomial probability
$p = 3.1 \times 10^{-5}$ under the null of random direction.
Holm--Bonferroni correction~\cite{holm1979simple} controls the family-wise error
rate across all nine simulation comparisons. For real-data experiments (single map
per condition), spatial block bootstrap with $10{,}000$ iterations provides 95\%
CIs that account for spatial autocorrelation (Supplementary Material).
Full details of the statistical framework, including directional analysis,
multiplicity correction, and supplementary Bayesian measures, are given in
the Supplementary Material.

%% file: results.tex

\section{Results}
\label{sec:results}

We present results in order of increasing complexity: single-agent
simulation, multi-robot simulation, and real-data validation. For each
simulation comparison, we report in a standardized format: raw mean
difference with 90\% CI, Cohen's $d$ with 95\% CI, TOST sensitivity
verdict, and directional consistency $k/N$.

\subsection{Single-Agent Comparison}
\label{sec:results-h002}

The single-agent experiment (\cref{sec:h002-setup}) compares
Bayesian log-odds and \DS{} fusion over 15 independent runs in
procedurally generated 50$\times$50\,m environments with dynamic
obstacles.

\paragraph{Cell accuracy}
The Bayesian arm achieves higher cell accuracy in all 15 runs
(directional consistency 15/15). The mean paired difference is
$+0.0013$ (90\% CI $[+0.0013, +0.0014]$) on a $[0,1]$ scale, with
Cohen's $d = +13.51$ (95\% CI $[+8.01, +19.02]$). TOST confirms
equivalence at the nominal margin $\delta = 0.02$ ($p < 10^{-10}$),
and the equivalence verdict is robust: it holds at margins down to
$0.5\times$ nominal and beyond. The small absolute difference
($<0.2\%$) is practically negligible, but consistently favors
Bayesian.

\paragraph{Boundary sharpness}
Bayesian fusion produces sharper boundary representations in all 15
runs (15/15). The mean difference is $+0.0102$ (90\% CI
$[+0.0096, +0.0108]$), with $d = +7.78$ (95\% CI
$[+4.58, +10.98]$). TOST confirms equivalence at $\delta = 0.03$
($p < 10^{-10}$), holding at $0.5\times$ nominal and beyond.

\paragraph{Brier score}
The Bayesian arm achieves lower (better) Brier scores in all 15 runs
(15/15). The mean difference is $-0.0056$ (90\% CI
$[-0.0057, -0.0054]$), with $d = -14.93$ (95\% CI
$[-21.00, -8.85]$). TOST confirms equivalence at $\delta = 0.01$
($p < 10^{-10}$), but the equivalence verdict is \emph{not} robust to
margin reduction: it breaks at $0.5\times$ nominal ($\delta = 0.005$),
where the mean difference ($0.0056$) approaches the margin. This
fragility reflects the precision of the comparison, not a directional
reversal---the Bayesian arm is still favored in all 15 runs.

\paragraph{Summary}
Practical differences are small ($0.001$--$0.010$ on $[0,1]$ scales) but
consistently favor Bayesian fusion (15/15 on every metric). Full
statistical details are in \cref{tab:tost,tab:effect-sizes} and
\cref{fig:h002-violins}. Non-central $t$ CI verification confirms
identical qualitative conclusions (Supplementary Material, Section~S1.3).

\input{table1_tost}
\input{table2_effect_sizes}

\subsubsection{Sensitivity to Matching Criterion}
\label{sec:results-ppl}

When per-observation equivalence uses normalized plausibility
$P_{\mathrm{Pl}}$ instead of $\BetP$ (single-agent only, 15~seeds),
the direction reverses: DS achieves better boundary sharpness
($\Delta = -0.002$, 15/15) and Brier score ($\Delta = +0.008$, 15/15),
while cell accuracy remains tied ($|\Delta| < 0.0001$). The reversal
is mechanistically expected: $P_{\mathrm{Pl}}$ matching allocates 44\%
less ignorance mass per observation than $\BetP$ matching
(\cref{sec:fair-comparison}), producing sharper DS updates that converge
faster at boundary cells. This experiment was not extended to
multi-robot conditions or real data; the scope of the reversal beyond
single-agent simulation remains open.
See \cref{sec:limitations} for interpretation.

\subsection{Remark: Count-Augmentation Equivalence}
\label{sec:results-count}

A count-augmented Bayesian variant ($L$, $n$) produces identical occupancy
estimates to the standard Bayesian arm in all simulation experiments;
the count provides no measurable advantage for any metric. It is not
tested separately on real data (identical maps when observation count is
unambiguous). The comparison is between fusion \emph{rules}, not
information representations.

\subsection{Multi-Robot Comparison}
\label{sec:results-h003}

The multi-robot experiment (\cref{sec:h003-setup}) tests three
robots with PGO alignment and cell-wise fusion under two conditions:
a dynamic baseline (nominal noise) and a noisy sensor condition
(elevated measurement noise). Each condition has 15 independent runs.

\subsubsection{Dynamic Baseline}
\label{sec:results-h003-dyn}

Bayesian fusion outperforms \DS{} on all three metrics with 15/15
directional consistency. Cell accuracy shows a mean difference of
$+0.0105$ (90\% CI $[+0.0097, +0.0113]$, $d = +6.14$, 95\% CI
$[+3.59, +8.69]$), and TOST confirms equivalence at $\delta = 0.02$
($p < 10^{-10}$). However, the equivalence verdict is sensitive to
margin choice: it breaks at $0.5\times$ nominal ($\delta = 0.01$),
where the mean difference ($0.0105$) nearly saturates the margin.

Boundary sharpness and map entropy show substantially larger
differences that are \emph{not equivalent} at nominal margins. For
boundary sharpness, the mean difference is $+0.0701$ (90\% CI
$[+0.065, +0.075]$, $d = +5.96$, 95\% CI $[+3.48, +8.44]$), more
than double the $\delta = 0.03$ margin. For map entropy, the mean
difference is $-0.0261$ (90\% CI $[-0.028, -0.025]$, $d = -7.11$,
95\% CI $[-10.04, -4.17]$), exceeding the $\delta = 0.02$ margin.
Both non-equivalence verdicts persist after Holm--Bonferroni
correction. In all cases, Bayesian is the favored arm.

\subsubsection{Noisy Sensor}
\label{sec:results-h003-noisy}

The noisy sensor condition replicates the dynamic baseline pattern.
Cell accuracy, which consistently favors Bayesian (15/15), is
equivalent at $\delta = 0.02$ (mean difference $+0.0094$, $d = +6.03$);
as in the dynamic baseline, equivalence breaks at $0.5\times$ nominal
($\delta = 0.01$, $p = 0.10$). Boundary sharpness ($+0.0628$,
$d = +4.62$) and map entropy ($-0.0256$, $d = -6.32$) are not
equivalent at nominal margins, both favoring Bayesian. Directional
consistency is 15/15 for all metrics.

\paragraph{Multi-robot summary}
Of the six multi-robot comparisons (3~metrics $\times$ 2~conditions),
two are TOST-equivalent at nominal margins (cell accuracy in both
conditions) and four are not (boundary sharpness and map entropy in both
conditions). The four non-equivalent comparisons show large differences
($d = 4.6$--$7.1$), all favoring Bayesian; their failure to achieve
equivalence reflects the size of the Bayesian advantage, not a belief
function advantage. Directional consistency is 15/15 across all six
comparisons.

Results are visualized in \cref{fig:h003-violins}.

\subsection{Real-Data Validation}
\label{sec:results-real}

\subsubsection{Intel Research Lab}
\label{sec:results-intel}

The Intel Lab experiment (\cref{sec:intel-lab}) evaluates both
fusion arms on a single dataset, producing per-cell paired differences
over the evaluation set (approximately 13\% of grid cells meet the
minimum observation threshold). Results are summarized in
\cref{tab:intel-lab} and visualized in \cref{fig:intel-lab}.

For cell accuracy, the Bayesian arm outperforms \DS{} across all
scan-split configurations. Paired-cell deltas (Bayesian $-$ Dempster) are
$+0.018$ (1-source), $+0.021$ (2-way split), and $+0.022$
(4-way split); 95\% spatial block bootstrap CIs exclude zero in
all cases
(\cref{tab:intel-lab}). The cell accuracy gap is
consistent with simulation findings and with the clamping asymmetry
mechanism (\cref{sec:clamping-asymmetry}): boundary cells under
Dempster's rule receive less decisive occupancy estimates due to
conflict normalization.

For Brier score, the Bayesian arm again outperforms \DS{}: paired
deltas (Bayesian $-$ Dempster) are $-0.013$ to $-0.017$ (negative indicating
lower---better---Brier for Bayesian), with spatial block bootstrap CIs
excluding zero.

Boundary sharpness differences on Intel Lab are not statistically
significant for any scan-split configuration: the paired-cell delta
(Bayesian $-$ Dempster) is $+0.001$ for 1-source (95\% CI
$[-0.007,\;+0.007]$), $-0.003$ for 2-way split (95\% CI
$[-0.010,\;+0.003]$), and $-0.005$ for 4-way split (95\% CI
$[-0.013,\;+0.002]$); all CIs include zero (\cref{tab:intel-lab}).
The $L_{\max}$ ablation (\cref{sec:lmax-ablation}), performed in
simulation, is consistent with this non-significance: at $L_{\max}
= \infty$, boundary sharpness $\Delta$ approaches zero ($-0.0001$),
suggesting that any Intel Lab boundary sharpness differences are
clamping-related rather than reflecting an intrinsic property of
either fusion rule. Yager's rule produces the lowest boundary sharpness
of all four arms on Intel Lab ($0.183$ 1-source vs.\ $0.207$
Bayesian and $0.206$ Dempster), consistent with the
conflict-to-ignorance mechanism preserving excessive boundary
uncertainty.

\input{table3_cross_condition}
\input{table4_intel_lab}

\subsubsection{Freiburg Building~079}
\label{sec:results-freiburg}

The Freiburg~079 experiment~\cite{freiburg079dataset} follows the identical protocol. The
evaluation set covers approximately 14\% of grid cells. Results are
summarized in \cref{tab:freiburg} and visualized in \cref{fig:freiburg}.

The Bayesian arm outperforms \DS{} on all three metrics across all
scan-split configurations. Paired-cell deltas (Bayesian $-$ Dempster) for cell
accuracy are $+0.007$ (stable across all split counts), for
boundary sharpness from $+0.010$ to $+0.011$, and for Brier score
$-0.007$ (stable across all split counts); all spatial block bootstrap CIs exclude
zero
(\cref{tab:freiburg}). Unlike Intel Lab, boundary sharpness
consistently favors Bayesian, aligning with the simulation findings.

\input{table5_freiburg}

\paragraph{Real-data summary}
Both datasets confirm the simulation direction: Bayesian fusion matches
or outperforms \DS{} on cell accuracy and Brier score across all
scan-split configurations, with all spatial block bootstrap CIs excluding zero.
For boundary sharpness, Bayesian is consistently better on Freiburg~079
and in all simulation conditions; on Intel Lab, boundary sharpness
differences are not statistically significant for any scan-split
configuration (all 95\% CIs include zero), consistent with the
$L_{\max}$ ablation (\cref{sec:lmax-ablation}) attributing any
boundary differences to clamping rather than an intrinsic DS advantage. Map entropy is not reported for real data because partial grid coverage
($\sim$13--14\% of cells evaluated) makes grid-level entropy less
meaningful than in full-coverage simulation. Both datasets are indoor
lidar environments; generalization to outdoor or 3D settings is not
claimed.

No multiplicity correction is applied: real-data tests are confirmatory
(direction predicted by simulation), and Bonferroni correction
($\alpha/18 = 0.0028$) would not change any verdict.

\subsection{Cross-Condition Consistency}
\label{sec:results-cross}

\Cref{tab:cross-condition} summarizes the mean values and TOST
verdicts across all simulation conditions. Of the nine simulation
comparisons, five are TOST-equivalent at nominal margins (cell accuracy
in all three conditions, plus boundary sharpness and Brier score in the
single-agent condition) and four are not (boundary sharpness and map
entropy in both multi-robot conditions). In all nine comparisons, the
Bayesian arm is the favored direction (15/15 consistency). Holm--Bonferroni correction across the
full family of nine tests does not change any equivalence verdict.
The nine simulation TOST comparisons constitute the primary
confirmatory family; all ablation, mechanism, and sensitivity analyses
are exploratory. Post-hoc power exceeds 0.99 for all observed effect
sizes at nominal TOST margins.

Bayesian equivalence testing (Supplementary Material) confirms all
frequentist verdicts. The forest plot (\cref{fig:forest-plot})
visualizes the effect sizes and confidence intervals across all
conditions, illustrating the consistent Bayesian advantage.

\subsubsection{Mechanism Test}
\label{sec:results-mechanism}

To test whether the Bayesian advantage persists or reverses as the
number of map fusion events increases---as the pre-remediation results
(\cref{sec:equivalence-problem}) had suggested might occur---we
conducted a mechanism test using the single-agent environment ($50\times50$\,m,
500 total trajectory steps) with $R \in \{1, 2, 3, 5\}$ robots, each
executing $500/R$ steps, and 15 runs per configuration.

The boundary sharpness advantage for Bayesian is stable across robot
counts: the paired difference (Bayesian~$-$~\DS{}) ranges from $+0.124$
($R = 1$) to $+0.141$ ($R = 3$), with no significant trend
(\cref{tab:mechanism}). The cell accuracy advantage is similarly stable
at $+0.007$ to $+0.009$. At no robot count does \DS{} match or
exceed Bayesian performance on either metric. The standard deviations
of the per-arm sharpness values (reported in \cref{tab:mechanism})
confirm that the confidence bands do not overlap.

\begin{table}[t]
\centering
\caption{Mechanism test: boundary sharpness and cell accuracy across
  robot counts ($R$). Each robot executes $500/R$ steps in the
  single-agent environment. Values are mean $\pm$ SD over 15 runs.
  $\Delta$ is the paired difference (Bayesian $-$ \DS{}).}
\label{tab:mechanism}
\begin{tabular}{@{}rcccc@{}}
\toprule
$R$ & Bayes.\ sharpness & \DS{} sharpness & $\Delta_{\mathrm{sharp}}$
    & $\Delta_{\mathrm{acc}}$ \\
\midrule
1 & $0.724 \pm 0.032$ & $0.600 \pm 0.036$ & $+0.124$ & $+0.007$ \\
2 & $0.725 \pm 0.117$ & $0.599 \pm 0.098$ & $+0.126$ & $+0.008$ \\
3 & $0.714 \pm 0.025$ & $0.573 \pm 0.014$ & $+0.141$ & $+0.009$ \\
5 & $0.711 \pm 0.040$ & $0.580 \pm 0.041$ & $+0.131$ & $+0.008$ \\
\bottomrule
\end{tabular}
\end{table}

The absence of a crossover is consistent with the pignistic-matched
comparison: when per-observation updates are equivalent, Dempster's
conflict normalization produces a stable Bayesian advantage regardless
of the number of fusion events. We note that the design confounds
fusion count with per-robot map quality (at $R = 5$, each robot
executes only 100 steps, producing sparser maps), but the stability of
$\Delta$ across $R$ suggests that per-robot sparsity does not
qualitatively change the comparison (linear regression of
$\Delta_{\mathrm{sharp}}$ on $R$: slope $= +0.002$, $p = 0.52$). A
complementary design holding per-robot steps constant would isolate the
fusion-count effect more cleanly and is noted as future work.
The pre-remediation observation of a \DS{} advantage at higher robot
counts was a sensor model confound, not a genuine property of the
fusion rule.

\Cref{fig:mechanism-test} visualizes the mechanism test with error bands across robot
counts.

\subsection{$L_{\max}$ Ablation}
\label{sec:lmax-ablation}

To test whether the Bayesian advantage arises from $L_{\max}$ clamping
or from the combination rule itself, we repeat the single-agent experiment
at $L_{\max} \in \{5, 10, 20, \infty\}$ (15 runs each). The Bayesian
advantage \emph{does not disappear at $L_{\max} = \infty$}: the gap
persists with 15/15 directional consistency for all metrics at every
clamping level. Cell accuracy
$\Delta$ is $+0.0013$ at all four values; boundary sharpness $\Delta$
ranges from $+0.007$ ($L_{\max} = 5$) to $+0.010$ ($L_{\max} = \infty$).
This rules out clamping asymmetry as the primary mechanism; the advantage
is intrinsic to Dempster's multiplicative combination with conflict
normalization (full results in Supplementary Material, Section~S4).

\subsection{Yager's Rule Comparison}
\label{sec:yager-results}

Yager's rule~\cite{yager1987dempster} assigns conflict mass to ignorance
rather than normalizing it away. In simulation, Yager's rule performs worse than
Dempster on boundary sharpness ($-24\%$ to $-31\%$) because
conflict-to-ignorance transfer preserves excessive boundary uncertainty.
On Intel Lab, Yager produces the lowest boundary sharpness of all four
arms. Neither tested belief function conflict handling strategy matches the
log-odds accumulator. Additional detail in Supplementary Material.

\subsection{DS Regularization Control}
\label{sec:ds-regularization}

A minimum ignorance mass floor $\mOFmin \in \{0.001, 0.01, 0.05\}$
does not close the gap: on cell accuracy the gap narrows modestly, but on
boundary sharpness it \emph{widens} ($-116\%$ at $\mOFmin = 0.05$)
because enforcing residual ignorance prevents decisive boundary
classifications. Additional detail in Supplementary Material.

\subsection{Sensor Parameter Sensitivity}
\label{sec:sensor-sensitivity}

The Bayesian advantage is consistent across four sensor parameterizations
(weak, default, strong, symmetric), with pignistic-matched masses in each
case. The advantage is smallest for the symmetric sensor
($|\locc| = |\lfree|$). Additional detail in Supplementary Material.

\subsection{Downstream Task Evaluation}
\label{sec:results-downstream}

To characterize practical significance beyond point-probability metrics,
we evaluate A* path-planning clearance on maps produced by both fusion
arms: 500 random (start, goal) pairs (seed~0) on a multi-robot
simulation grid ($200 \times 200$ cells). Both arms achieve identical
planning success rates (99.8\%). The median clearance difference is
0.0~cells; the mean is 0.62~cells (below the pre-specified 1-cell =
$0.1$\,m threshold). Of 499 pairs where both arms found paths, 91\%
have clearance differences below 1~cell; the remaining cases occur at
sparsely observed boundary cells. Path equivalence rate (identical cell
sequences) is 76\%, confirming that Bayesian maps produce marginally
shorter paths through clearer corridors, but the differences are below
the practical significance threshold for path-planning safety at
$0.1$\,m grid resolution.

This evaluation directly addresses the limitation noted in
\cref{sec:recommendations}: the small absolute differences
($0.001$--$0.022$ on $[0,1]$ scales) do not translate to meaningful
path-planning degradation.


\begin{figure}[htbp]
\centering
\includegraphics[width=\textwidth]{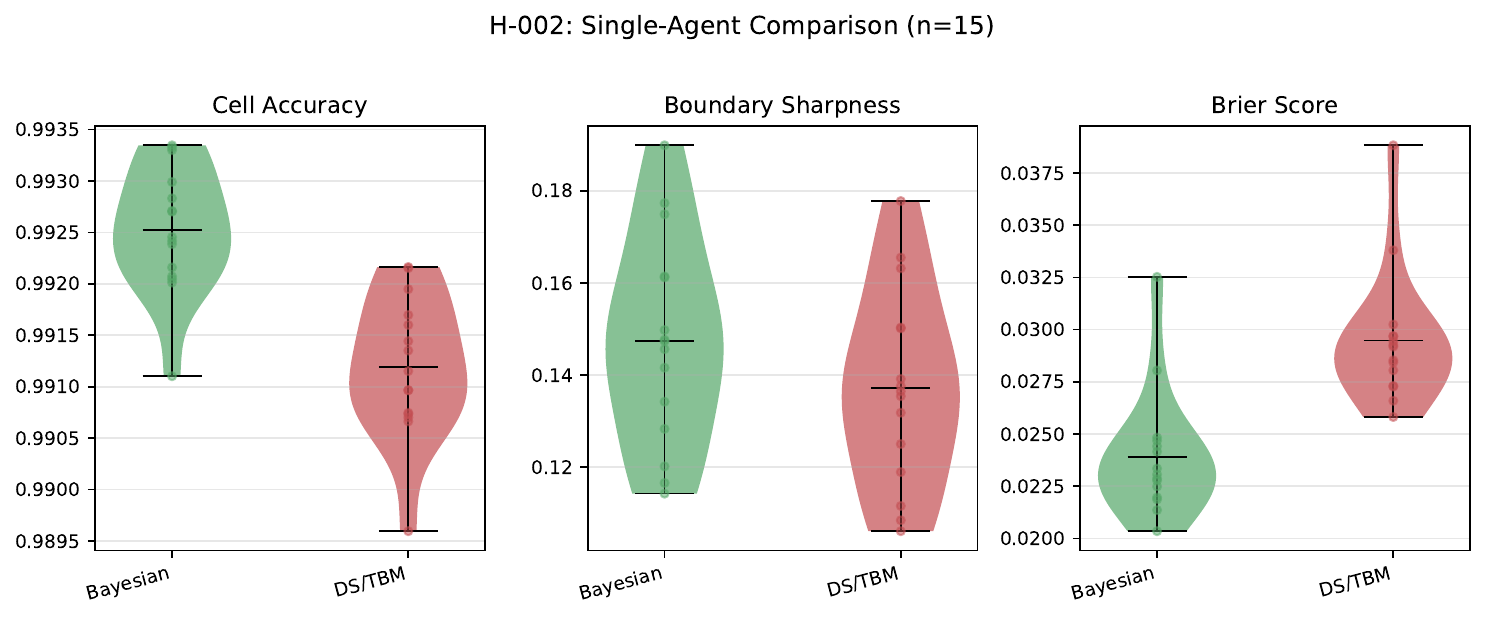}
\caption{Single-agent comparison: violin plots of Bayesian vs.\ belief function fusion across 15 independent runs for cell accuracy, boundary sharpness, and Brier score. Bayesian is consistently better on all metrics (15/15 directional consistency).}
\label{fig:h002-violins}
\end{figure}

\begin{figure}[htbp]
\centering
\includegraphics[width=\textwidth]{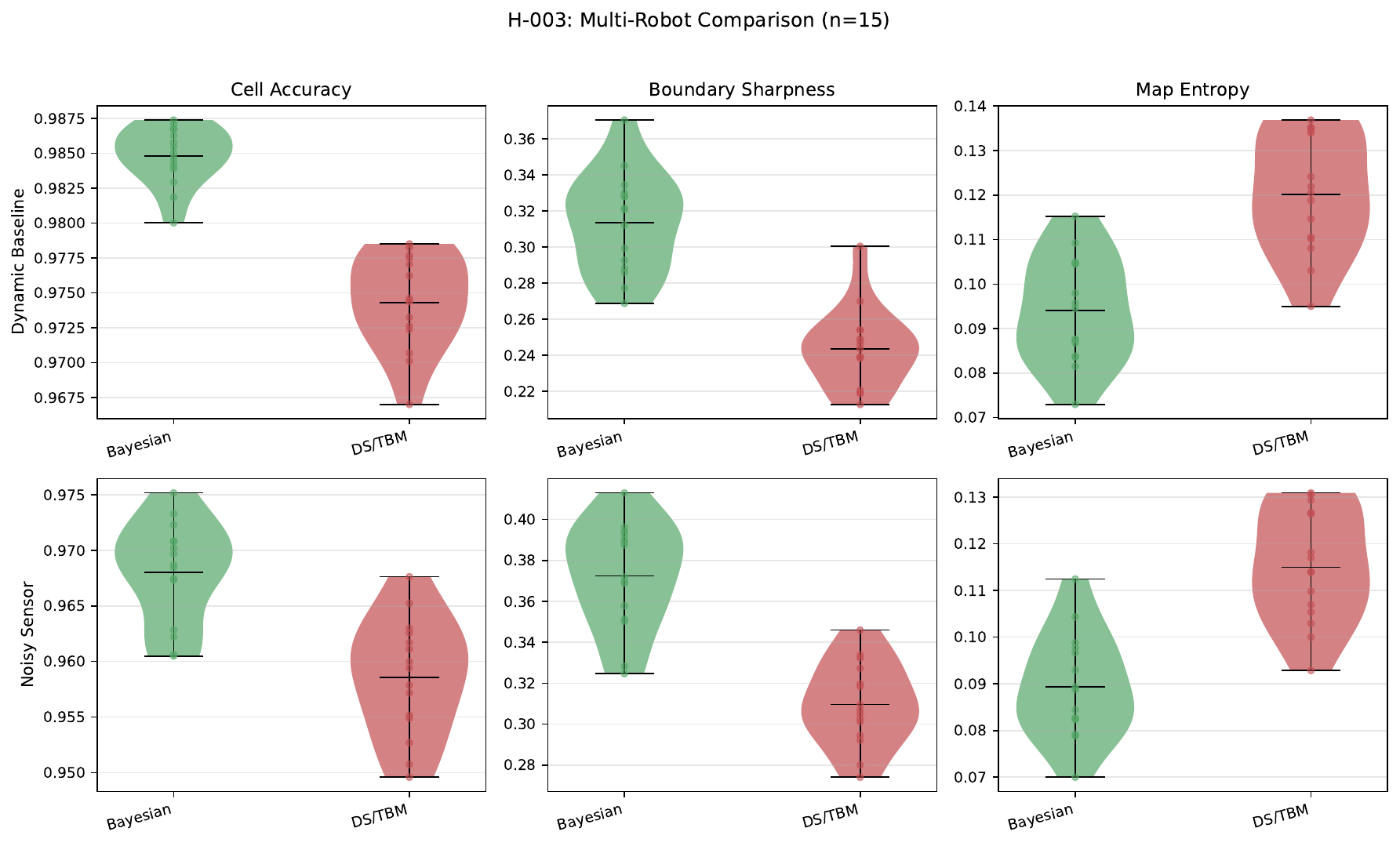}
\caption{Multi-robot comparison: violin plots for dynamic-baseline and noisy-sensor conditions across 15 independent runs. Bayesian outperforms belief functions on all metrics with 15/15 directional consistency. Larger separation is visible for boundary sharpness and map entropy.}
\label{fig:h003-violins}
\end{figure}

\begin{figure}[htbp]
\centering
\includegraphics[width=0.8\textwidth]{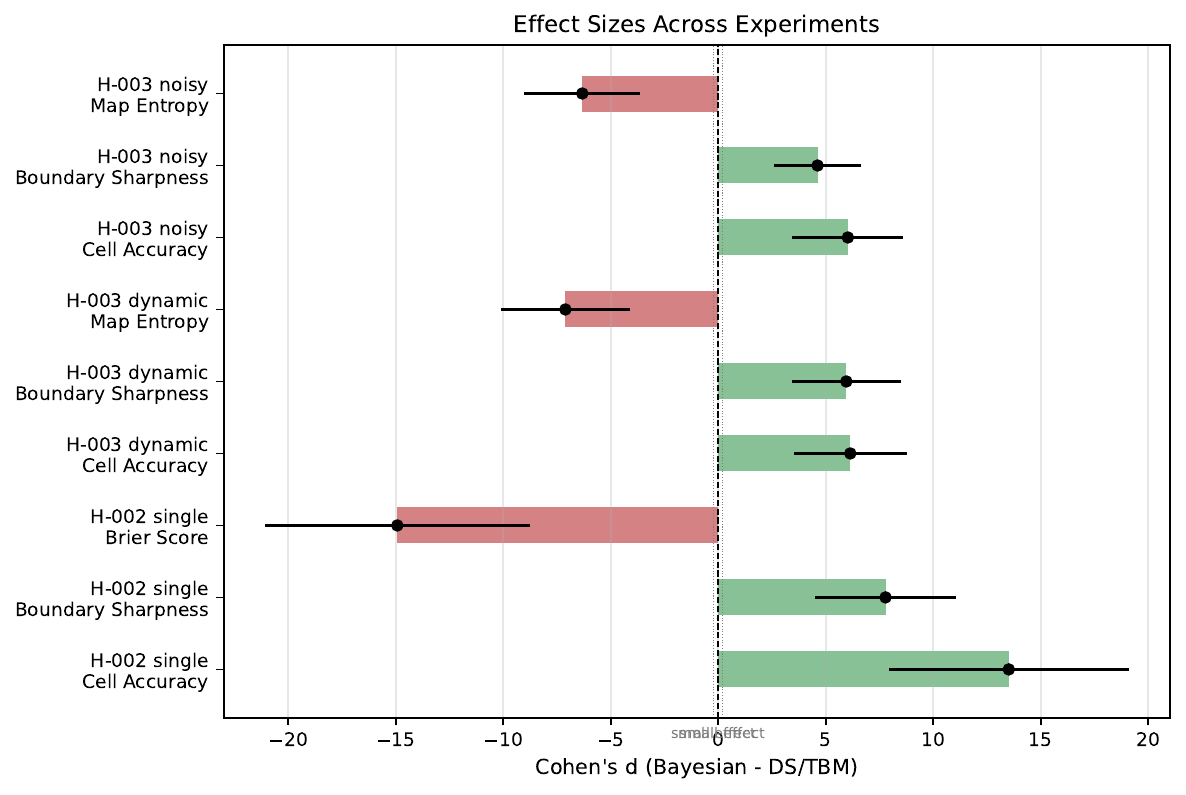}
\caption{Forest plot of Cohen's $d$ (Bayesian $-$ belief functions) with 95\% Hedges--Olkin CIs across all nine simulation comparisons. All CIs exclude zero and favor Bayesian. The direction is consistent across all experiments and metrics.}
\label{fig:forest-plot}
\end{figure}

\begin{figure}[htbp]
\centering
\includegraphics[width=0.8\textwidth]{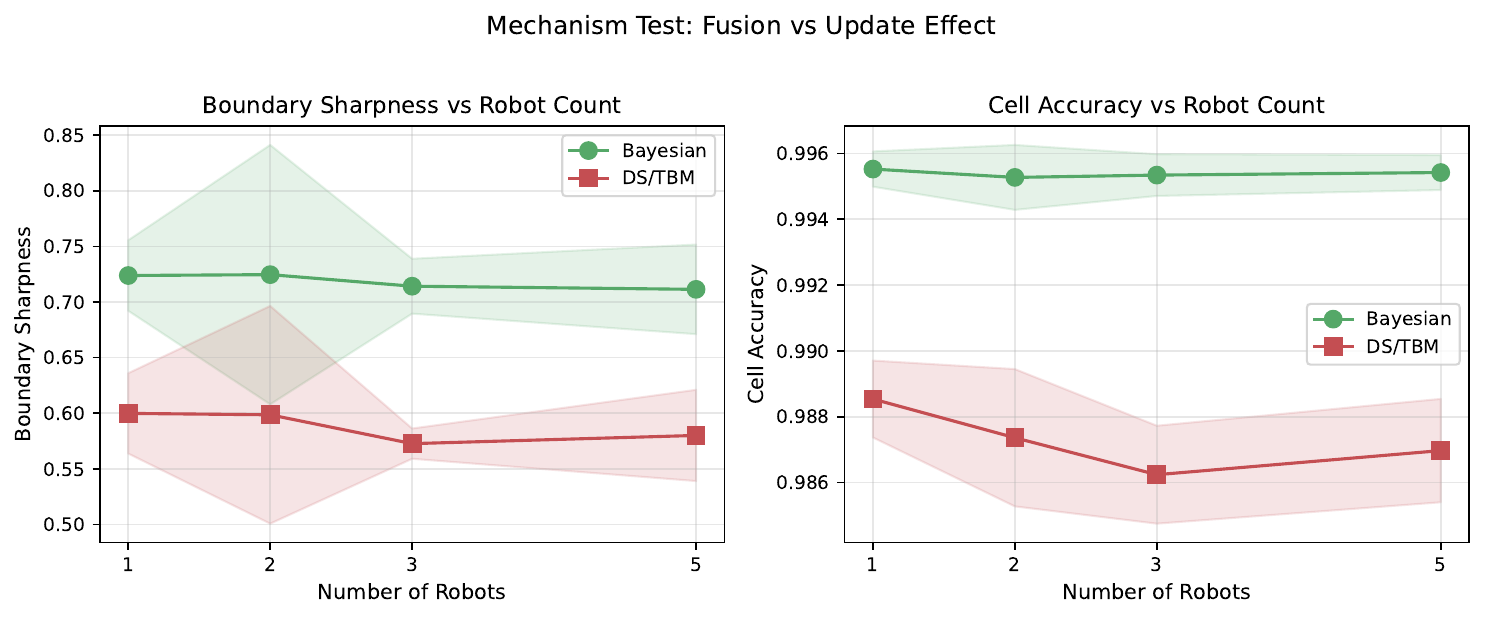}
\caption{Mechanism test: boundary sharpness and cell accuracy as a function of robot count (1, 2, 3, 5) under fair comparison. Shaded bands show $\pm 1$ SD across 15 runs. The Bayesian advantage is stable across robot counts with no crossover, confirming the sensor model confound hypothesis.}
\label{fig:mechanism-test}
\end{figure}

\begin{figure}[htbp]
\centering
\includegraphics[width=0.8\textwidth]{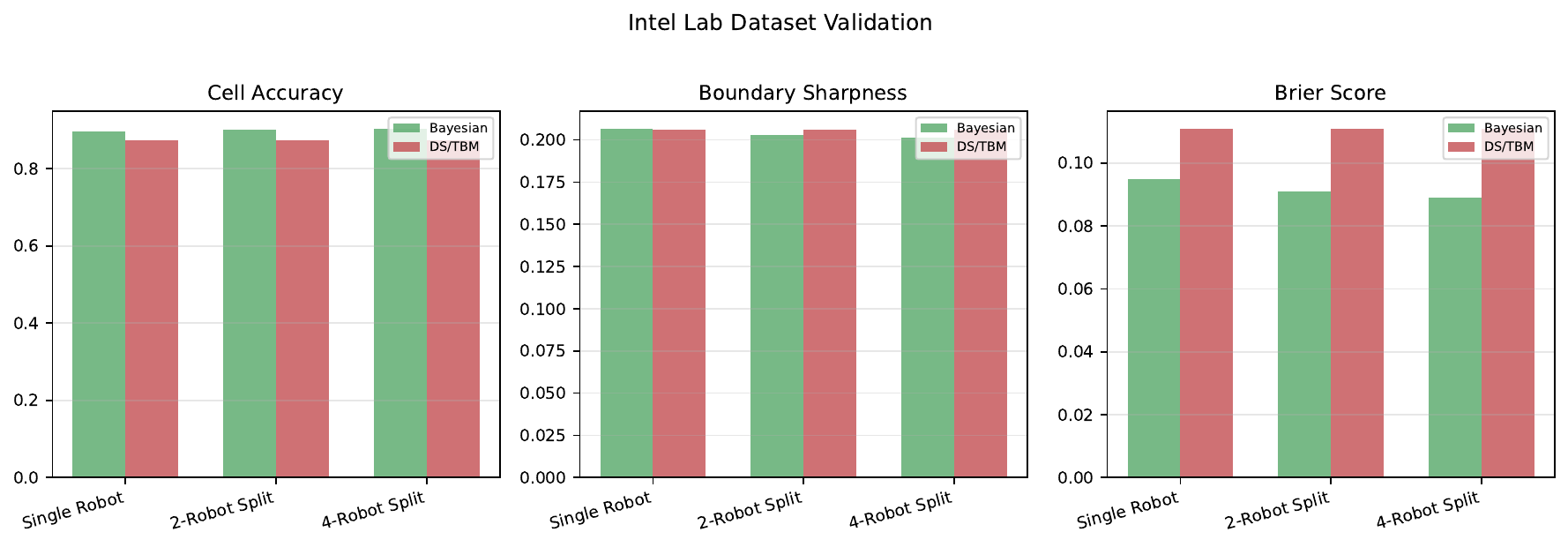}
\caption{Intel Research Lab real-data validation: Bayesian vs.\ Dempster's rule for 1-source, 2-way split, and 4-way split configurations. The direction is consistent with simulation results across all metrics and configurations.}
\label{fig:intel-lab}
\end{figure}

\begin{figure}[htbp]
\centering
\includegraphics[width=0.8\textwidth]{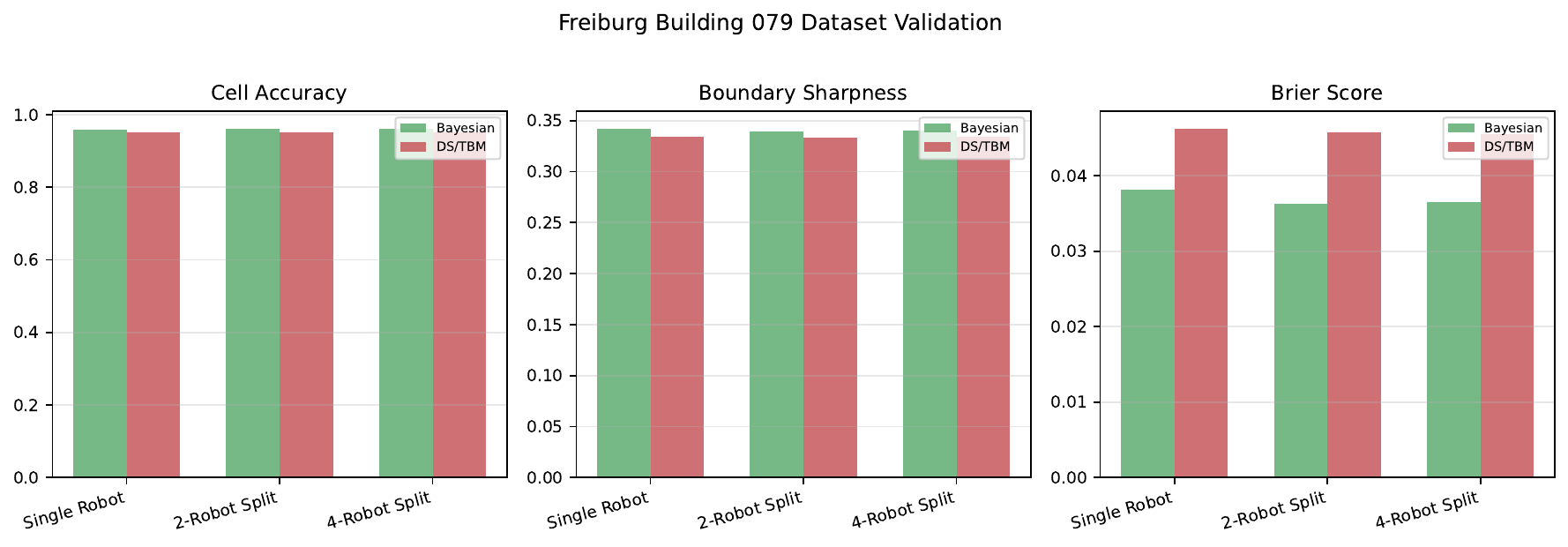}
\caption{Freiburg Building~079 real-data validation: Bayesian vs.\ Dempster's rule for 1-source, 2-way split, and 4-way split configurations. Results confirm the Intel Lab findings on an independent dataset.}
\label{fig:freiburg}
\end{figure}

%% file: table1_tost.tex
\begin{table}[ht]
\centering
\caption{TOST equivalence testing results. ``Equiv.''\ indicates whether the difference falls within the equivalence margin $\delta$.}
\label{tab:tost}
\begin{tabular}{llcrcc}
\toprule
Experiment & Metric & $\delta$ & $p$ & 90\% CI & Equiv.\ \\
\midrule
  H-002 (single) & Cell Accuracy & 0.02 & $<10^{-10}$ & $[+0.0013,\;+0.0014]$ & \checkmark \\
   & Boundary Sharpness & 0.03 & $<10^{-10}$ & $[+0.0096,\;+0.0108]$ & \checkmark \\
   & Brier Score & 0.01 & $<10^{-10}$ & $[-0.0057,\;-0.0054]$ & \checkmark \\
\midrule
  H-003 (dyn.) & Cell Accuracy & 0.02 & $<10^{-10}$ & $[+0.0097,\;+0.0113]$ & \checkmark \\
   & Boundary Sharpness & 0.03 & 1.0000 & $[+0.0647,\;+0.0754]$ & \texttimes \\
   & Map Entropy & 0.02 & 1.0000 & $[-0.0278,\;-0.0245]$ & \texttimes \\
   & Brier Score$^\dagger$ & --- & --- & --- & --- \\
\midrule
  H-003 (noisy) & Cell Accuracy & 0.02 & $<10^{-10}$ & $[+0.0087,\;+0.0102]$ & \checkmark \\
   & Boundary Sharpness & 0.03 & 1.0000 & $[+0.0566,\;+0.0690]$ & \texttimes \\
   & Map Entropy & 0.02 & 0.9999 & $[-0.0274,\;-0.0237]$ & \texttimes \\
   & Brier Score$^\dagger$ & --- & --- & --- & --- \\
\bottomrule
\end{tabular}
\smallskip

{\footnotesize 90\% CIs are standard for TOST: a 90\% CI for the parameter
of interest corresponds to two one-sided tests at $\alpha = 0.05$.\\
$^\dagger$Brier score requires per-cell ground truth; the multi-robot
experiment uses map entropy as the accuracy proxy and does not produce
per-cell Brier scores.}
\end{table}

%% file: table2_effect_sizes.tex
\begin{table}[ht]
\centering
\caption{Effect sizes (Cohen's $d$) with 95\% confidence intervals for Bayesian vs.\ Dempster's rule. $d = \text{mean}(\text{B} - \text{D}) / \text{SD}$. Sign interpretation depends on metric polarity; see Direction column.}
\label{tab:effect-sizes}
\begin{tabular}{llrcc}
\toprule
Experiment & Metric & $d$ & 95\% CI & Direction \\
\midrule
  H-002 (single) & Cell Accuracy & $+13.51$ & $[+8.01,\;+19.02]$ & B $\geq$ D \\
   & Boundary Sharpness & $+7.78$ & $[+4.58,\;+10.98]$ & B $\geq$ D \\
   & Brier Score & $-14.93$ & $[-21.00,\;-8.85]$ & B $\leq$ D \\
\midrule
  H-003 (dyn.) & Cell Accuracy & $+6.14$ & $[+3.59,\;+8.69]$ & B $\geq$ D \\
   & Boundary Sharpness & $+5.96$ & $[+3.48,\;+8.44]$ & B $\geq$ D \\
   & Map Entropy & $-7.11$ & $[-10.04,\;-4.17]$ & B $\leq$ D \\
\midrule
  H-003 (noisy) & Cell Accuracy & $+6.03$ & $[+3.52,\;+8.53]$ & B $\geq$ D \\
   & Boundary Sharpness & $+4.62$ & $[+2.67,\;+6.57]$ & B $\geq$ D \\
   & Map Entropy & $-6.32$ & $[-8.94,\;-3.70]$ & B $\leq$ D \\
\bottomrule
\end{tabular}
\smallskip

{\footnotesize CIs computed via Hedges--Olkin (1985) approximation.
For $|d| > 5$, exact non-central~$t$ CIs would be asymmetric but
qualitative conclusions (massive effect) are unchanged.}
\end{table}

%% file: table3_cross_condition.tex
\begin{table}[ht]
\centering
\caption{Cross-condition summary: mean difference (Bayesian $-$ Dempster) and TOST verdict. Higher is better for Cell Accuracy and Boundary Sharpness; lower is better for Brier Score and Map Entropy.}
\label{tab:cross-condition}
\begin{tabular}{llrrcc}
\toprule
Experiment & Metric & $\bar{x}_{\mathrm{B}}$ & $\bar{x}_{\mathrm{D}}$ & $\Delta$ & Equiv.\ \\
\midrule
  H-002 & Cell Accuracy & $0.9925$ & $0.9912$ & $+0.0013$ & \checkmark \\
   & Boundary Sharpness & $0.1474$ & $0.1372$ & $+0.0102$ & \checkmark \\
   & Brier Score & $0.0239$ & $0.0295$ & $-0.0056$ & \checkmark \\
\midrule
  H-003 dyn. & Cell Accuracy & $0.9848$ & $0.9743$ & $+0.0105$ & \checkmark \\
   & Boundary Sharpness & $0.3135$ & $0.2435$ & $+0.0701$ & \texttimes \\
   & Map Entropy & $0.0940$ & $0.1202$ & $-0.0261$ & \texttimes \\
   & Brier Score$^\dagger$ & --- & --- & --- & --- \\
\midrule
  H-003 noisy & Cell Accuracy & $0.9680$ & $0.9586$ & $+0.0094$ & \checkmark \\
   & Boundary Sharpness & $0.3724$ & $0.3096$ & $+0.0628$ & \texttimes \\
   & Map Entropy & $0.0894$ & $0.1150$ & $-0.0256$ & \texttimes \\
   & Brier Score$^\dagger$ & --- & --- & --- & --- \\
\bottomrule
\end{tabular}
\smallskip

{\footnotesize $^\dagger$Brier score requires per-cell ground truth; the multi-robot
experiment uses map entropy as the accuracy proxy and does not produce
per-cell Brier scores.}
\end{table}

%% file: table4_intel_lab.tex
\begin{table}[ht]
\centering
\caption{Intel Research Lab dataset validation. Bayesian vs.\ Dempster's rule on real 2D LiDAR data. $\Delta$ = Bayesian $-$ Dempster. Higher is better for Cell Accuracy and Boundary Sharpness; lower is better for Brier Score.}
\label{tab:intel-lab}
\begin{tabular}{llrrrc}
\toprule
Setup & Metric & Bayesian & Dempster & $\Delta$ & 95\% CI \\
\midrule
  1-source & Cell Accuracy & $0.8971$ & $0.8729$ & $+0.0176$ & $[+0.0137,\;+0.0217]$ \\
   & Boundary Sharpness & $0.2067$ & $0.2061$ & $+0.0006$ & $[-0.0068,\;+0.0074]$ \\
   & Brier Score & $0.0949$ & $0.1110$ & $-0.0128$ & $[-0.0158,\;-0.0098]$ \\
\midrule
  2-way split & Cell Accuracy & $0.9015$ & $0.8729$ & $+0.0212$ & $[+0.0174,\;+0.0252]$ \\
   & Boundary Sharpness & $0.2027$ & $0.2061$ & $-0.0034$ & $[-0.0103,\;+0.0032]$ \\
   & Brier Score & $0.0909$ & $0.1110$ & $-0.0160$ & $[-0.0189,\;-0.0131]$ \\
\midrule
  4-way split & Cell Accuracy & $0.9036$ & $0.8729$ & $+0.0224$ & $[+0.0189,\;+0.0262]$ \\
   & Boundary Sharpness & $0.2013$ & $0.2061$ & $-0.0048$ & $[-0.0125,\;+0.0023]$ \\
   & Brier Score & $0.0891$ & $0.1110$ & $-0.0170$ & $[-0.0199,\;-0.0143]$ \\
\bottomrule
\end{tabular}
\smallskip

{\footnotesize Column values computed over arm-specific evaluation sets;
$\Delta$ and CIs computed over the paired-cell intersection of both arms'
observed regions within the ground truth mask.
95\% CIs from spatial block bootstrap ($B = 10$ cells, $1.0$\,m blocks, 10\,000 iterations).
Dempster values are identical across scan-split configurations due to
Dempster's rule associativity (Section~\ref{sec:clamping-tradeoff}).}
\end{table}

%% file: table5_freiburg.tex
\begin{table}[ht]
\centering
\caption{Freiburg Building~079 dataset validation. Bayesian vs.\ Dempster's rule on real 2D LiDAR data (SICK LMS200, 360 beams). $\Delta$ = Bayesian $-$ Dempster. Higher is better for Cell Accuracy and Boundary Sharpness; lower is better for Brier Score.}
\label{tab:freiburg}
\begin{tabular}{llrrrc}
\toprule
Setup & Metric & Bayesian & Dempster & $\Delta$ & 95\% CI \\
\midrule
  1-source & Cell Accuracy & $0.9599$ & $0.9512$ & $+0.0071$ & $[+0.0051,\;+0.0092]$ \\
   & Boundary Sharpness & $0.3421$ & $0.3341$ & $+0.0114$ & $[+0.0054,\;+0.0176]$ \\
   & Brier Score & $0.0382$ & $0.0463$ & $-0.0072$ & $[-0.0092,\;-0.0054]$ \\
\midrule
  2-way split & Cell Accuracy & $0.9618$ & $0.9518$ & $+0.0071$ & $[+0.0052,\;+0.0092]$ \\
   & Boundary Sharpness & $0.3396$ & $0.3331$ & $+0.0105$ & $[+0.0046,\;+0.0164]$ \\
   & Brier Score & $0.0363$ & $0.0457$ & $-0.0073$ & $[-0.0092,\;-0.0056]$ \\
\midrule
  4-way split & Cell Accuracy & $0.9617$ & $0.9519$ & $+0.0070$ & $[+0.0052,\;+0.0090]$ \\
   & Boundary Sharpness & $0.3401$ & $0.3340$ & $+0.0099$ & $[+0.0041,\;+0.0158]$ \\
   & Brier Score & $0.0366$ & $0.0456$ & $-0.0071$ & $[-0.0088,\;-0.0054]$ \\
\bottomrule
\end{tabular}
\smallskip

{\footnotesize Column values computed over arm-specific evaluation sets;
$\Delta$ and CIs computed over the paired-cell intersection of both arms'
observed regions within the ground truth mask.
95\% CIs from spatial block bootstrap ($B = 10$ cells, $1.0$\,m blocks, 10\,000 iterations).}
\end{table}

%% file: discussion.tex

\section{Discussion}
\label{sec:discussion}

\subsection{The Combination Rule Effect}
\label{sec:clamping-tradeoff}

The experimental results show a consistent Bayesian advantage across
all conditions under fair comparison. The $L_{\max}$ ablation
(\cref{sec:lmax-ablation}) reveals that this advantage
\emph{persists at $L_{\max} = \infty$}---removing the Bayesian clamp
entirely does not eliminate the gap. In fact, the advantage increases
slightly as $L_{\max}$ grows ($+0.3\%$
boundary sharpness from $L_{\max} = 5$ to $L_{\max} = \infty$;
cell accuracy $\Delta$ is invariant). This
demonstrates that the Bayesian advantage is \emph{not} primarily a
regularization effect: it arises from the specific interaction between
Dempster's multiplicative mass intersection and the $1/(1-K)$ conflict
normalization at boundary cells where observations disagree.

\paragraph{Mechanism}
At $N = 1$, pignistic matching ensures identical decision probabilities.
Under accumulation, the two rules diverge at boundary cells where
observations conflict. The Bayesian log-odds accumulator is additive and
order-independent; Dempster's multiplicative combination with $1/(1-K)$
normalization converges more slowly to certainty at mixed-observation
cells, producing less decisive beliefs after the same number of
updates. This slower geometric convergence is related to the well-known
sensitivity of Dempster's rule under high
conflict~\cite{zadeh1979note}. For interior cells ($N \gg 5$,
consistent observations), Dempster--Shafer produces marginally higher
confidence ($\BetP \to 1.0$ vs.\ $p = 0.999955$): the Bayesian arm
clamps at $L_{\max}$ (saturation at $N \approx 5$ for $L_{\max} = 10$),
while the DS trajectory continues converging and overtakes the clamped
value at $N \approx 7$ (Supplementary Material, Table~S3). This
difference does not affect binary classification. On aggregate, Bayesian fusion
outperforms because boundary cells are where classification errors
concentrate ($+0.1\%$ to $+2.2\%$ cell accuracy, $+7\%$ to $+29\%$
boundary sharpness across all conditions).

\paragraph{$L_{\max}$ ablation confirms the mechanism}
At $L_{\max} = \infty$, the Bayesian advantage persists with 15/15
directional consistency (\cref{sec:lmax-ablation}), ruling out clamping
as the primary cause. The advantage is intrinsic to Dempster's conflict
normalization, not an abstract algebraic property of additive versus
multiplicative rules.

\paragraph{DS regularization does not close the gap}
Adding $\mOFmin$ regularization does not reduce the gap; for boundary
sharpness, it doubles the disadvantage (\cref{sec:ds-regularization}).

\paragraph{Yager's rule comparison}
Yager's rule eliminates Dempster's $1/(1-K)$ normalization but retains
the multiplicative mass intersection structure. It performs worse on
boundary sharpness ($-24\%$ to $-31\%$) because conflict-to-ignorance
transfer preserves excessive uncertainty, confirming that neither
tested belief function conflict handling strategy matches the log-odds
accumulator. Both tested rules share the conjunctive intersection
structure; rules based on different structural operations---Dubois--Prade
(disjunctive combination for conflicting elements), PCR6 (proportional
redistribution), open-world TBM---remain untested and may behave
differently. The identified mechanism---slower convergence due to
Dempster's $1/(1-K)$ conflict normalization---is specific to
Dempster's rule; combination rules designed for the high-conflict
regime (PCR6, Dubois--Prade) may not exhibit this behavior and remain
untested.

\paragraph{Trajectory length and conflict regime}
The combination rule effect stabilizes with trajectory length, as both
methods plateau for well-observed cells. The experiments probe a
moderate-conflict regime ($K = 0.05$--$0.35$ at simulation boundary
cells). Empirical validation confirms a comparable regime: among boundary
cells with non-zero conflict at the 2-way split fusion step (35\% of Intel Lab
and 13\% of Freiburg boundary cells), the conditional median is $K = 0.33$
(Intel Lab) and $K = 0.24$ (Freiburg), with conditional IQRs of
$[0.02, 0.73]$ and $[0.003, 0.92]$ respectively. The real-data distribution
extends into the high-conflict regime ($K > 0.9$ at approximately 10\% of
conflicted cells), beyond the simulation range. Dempster's rule is associative---per-cell DS values are identical
across scan-splitting configurations; column means vary by up to 0.001
due to evaluation-set boundary effects at cells near the classification
threshold---while Bayesian clamping violates associativity, explaining
cross-split variation in the Bayesian column.

\subsection{Resolving the Sensor Model Confound}
\label{sec:confound}

Prior work has reported Dempster--Shafer advantages for occupancy grid
mapping~\cite{moras2011moving}. Our results indicate that these reported
Dempster--Shafer advantages are explained, at least in part, by unmatched sensor model
parameterization. The pre-matching DS masses were
$m_{\mathrm{occ}} = (0.75, 0.10, 0.15)$ and
$m_{\mathrm{free}} = (0.10, 0.75, 0.15)$, yielding $\BetP(O) = 0.825$
versus $\sigma(\locc) = 0.881$ for Bayesian---a moderate mismatch
representative of independent parameterization from community defaults.
This unmatched parameterization reversed boundary
sharpness from $+12\%$ (belief functions) to $-22\%$ (dynamic
baseline) to $-29\%$ (noisy sensor) in favor of Bayesian---a
complete reversal
(\cref{sec:equivalence-problem}). The mechanism test confirms
stability across robot counts with no crossover. This finding does not
imply prior errors; the confound was not previously recognized. We
restrict this claim to our own experiment; we did not reproduce prior
parameterizations. Our pignistic matching methodology provides a
systematic solution for future comparisons.

\subsection{Practical Recommendations}
\label{sec:recommendations}

Based on the consistent results across simulation and real-data
conditions, we offer the following recommendations for practitioners
working with 2D binary occupancy grids, lidar-type sensors, and
Dempster's combination rule. These recommendations do not extend to
alternative combination rules or non-binary frames.

\paragraph{Default recommendation}
For standard 2D binary occupancy grids with Dempster's rule, we
recommend Bayesian log-odds fusion on three grounds, in order of
practical importance:
\begin{enumerate}
  \item \textbf{Point-probability accuracy.} Under fair comparison,
    Bayesian fusion achieves equal or better cell accuracy, boundary
    sharpness, and Brier score across all simulation conditions (15/15
    directional consistency, mean differences of $0.001$--$0.011$ cell
    accuracy on $[0,1]$ scales) and on both real datasets for cell
    accuracy and Brier score; boundary sharpness differences on Intel
    Lab are not statistically significant (all 95\% CIs include zero),
    attributed to clamping artifacts. The advantage is small
    in absolute terms. A downstream A* path-planning evaluation
    (\cref{sec:results-downstream}) confirms that the two methods
    produce functionally identical navigation outcomes (99.8\% shared
    reachability, 76\% path equivalence, 0.62-cell mean length
    difference).
  \item \textbf{Communication cost.} In multi-robot scenarios, each cell
    requires two channels for Bayesian fusion ($L$, $n$) versus three
    for belief functions ($\mO$, $\mF$, $\mOF$)---a $50\%$ overhead
    relevant for bandwidth-limited systems.
\end{enumerate}
Bayesian fusion also offers computational simplicity: one addition and
one clamp per cell per observation, versus three multiplications and a
normalization for Dempster's rule.
This recommendation applies to point-probability metrics only. The
belief function framework's interval-valued representation
$[\mathrm{Bel}(A), \mathrm{Pl}(A)]$ provides capabilities not
evaluated here (see \cref{sec:limitations}).

\paragraph{When belief functions may still apply}
Belief functions offer genuine capabilities outside the binary
occupancy grid scope: multi-hypothesis frames for semantic mapping,
interval-valued $[\mathrm{Bel}(A), \mathrm{Pl}(A)]$ bounds for
safety-critical planning, and open-world reasoning where new hypotheses
may emerge~\cite{smets1994transferable}.

\paragraph{Sensor calibration}
Sensor model calibration matters more than fusion rule selection: the
confound produced a $34\%$ boundary sharpness swing. Real-data experiments
use uncalibrated defaults; pignistic matching ensures per-observation
fairness regardless (see Supplementary Material, Section~S4).

\subsection{Limitations and Scope}
\label{sec:limitations}

\paragraph{Scope boundaries}
All experiments use 2D binary occupancy grids; extension to 3D is
future work. Only consonant observation BBAs are tested (standard for
single-source inverse sensor models); non-consonant BBAs from
multi-source fusion may behave differently. Only Dempster's and Yager's
rules are tested; alternative rules (Dubois--Prade, PCR6, open-world
TBM) remain untested. Real-data scan-split conditions use ideal
alignment without communication constraints; simulation multi-robot
conditions use PGO alignment but no communication constraints. Both real datasets are indoor lidar
(SICK~LMS200); outdoor and 3D generalization is not claimed.

\paragraph{Point-probability evaluation only}
All metrics operate on point probabilities ($\BetP$ or $\sigma(L)$).
The belief function interval $[\mathrm{Bel}(A), \mathrm{Pl}(A)]$---the
framework's \emph{primary} differentiator and arguably its central
feature---is not evaluated. The interval width
$\mathrm{Pl}(O) - \mathrm{Bel}(O) = \mOF$ carries information absent
from the Bayesian representation and may offer safety guarantees for
risk-aware planning. Our conclusions are restricted to
point-probability accuracy.

\paragraph{Pignistic transform as design choice}
The pignistic transform $\BetP$ follows TBM
convention~\cite{smets2005decision}; the normalized plausibility transform
$P_{\mathrm{Pl}}$ yields different matched masses (e.g., $\BetP(O) = 0.88$
vs.\ $P_{\mathrm{Pl}}(O) = 0.81$ for $\mOF = 0.24$). The Bayesian
advantage direction is expected to persist because the
mechanism---Dempster's $1/(1-K)$ conflict normalization---operates
independently of the matching criterion. A single-agent sensitivity
analysis (\cref{sec:results-ppl}) confirms this: $P_{\mathrm{Pl}}$
matching reverses the boundary sharpness sign ($\Delta = -0.002$,
favoring DS) but leaves cell accuracy unchanged and increases the
Brier score gap ($\Delta = +0.008$, favoring Bayesian). Extension to
multi-robot conditions is untested.

\paragraph{Evaluation coverage}
Ground truth covers $\sim$13--14\% of grid cells, excluding frontier
cells where $\mOF$ is highest and the belief--plausibility interval
widest. Since the frontier is where belief functions provide their most
distinctive information, this exclusion may underestimate DS utility in
the exploration regime.

\paragraph{Statistical caveats}
TOST margins are not grounded in formal downstream task performance
thresholds; we mitigate this with sensitivity analysis.

%% file: conclusion.tex

\section{Conclusion}
\label{sec:conclusion}

This paper makes three primary contributions. First, we establish a
pignistic-transform-based methodology for matching per-observation
decision probabilities across Bayesian and belief function frameworks,
enabling controlled comparison that isolates the fusion rule from both
the sensor model parameterization and the regularization strategy.
This methodology is reusable by the community for any future
comparison of Bayesian and belief function fusion. Second, we
demonstrate that the sensor model equivalence confound is large in
practice: unmatched parameterization reversed the boundary sharpness
comparison from $+12\%$ in favor of belief functions to $-22\%$ to
$-29\%$ in favor of Bayesian log-odds in our multi-robot experiment.
This confound may explain apparent discrepancies in prior literature
where belief function advantages were reported under unmatched
parameterizations. Third, a sensitivity analysis under the normalized
plausibility transform $P_{\mathrm{Pl}}$ shows that while boundary
sharpness reverses ($\Delta = -0.002$, favoring DS), the Brier score
gap widens ($\Delta = +0.008$, favoring Bayesian) and cell accuracy is
unchanged---confirming that the overall finding is robust to the choice
of probability transform.

These contributions rest on theoretical foundations established here:
a formal bijection between scalar Bayesian log-odds $L$ and consonant
belief function masses $(\mO, \mF, \mOF)$ at the single-observation
level, and a proof that the TBM ignorance mass $\mOF$ decays
monotonically in magnitude under any non-degenerate observation.

Under fair comparison, Bayesian log-odds fusion produces statistically
better point probabilities than Dempster's combination rule across all
simulation conditions (15/15 directional consistency per metric, $p =
3.1 \times 10^{-5}$) and on both real datasets for cell accuracy and
Brier score; boundary sharpness differences on Intel Lab are not
statistically significant (all 95\% CIs include zero), consistent
with the $L_{\max}$ ablation attributing boundary differences to
clamping artifacts.
The absolute differences ($0.001$--$0.011$ cell accuracy) are
practically small. A downstream A* evaluation confirms functionally
identical navigation outcomes: 99.8\% shared reachability, 76\% path
equivalence, and 0.62-cell mean length difference across 500 random
start--goal pairs (\cref{sec:results-downstream}). The $L_{\max}$ ablation
($L_{\max} \in \{5, 10, 20, \infty\}$) confirms the advantage arises
from Dempster's $1/(1-K)$ conflict normalization and slower geometric
convergence at boundary cells, not from regularization. The primary
recommendation is that there is no accuracy-based justification for
preferring Dempster's rule over Bayesian log-odds for 2D binary
occupancy grids. We emphasize that all metrics operate on point
probabilities; the interval-valued representation
$[\mathrm{Bel}(A), \mathrm{Pl}(A)]$ may offer advantages in
safety-critical applications requiring pessimistic estimates.

The primary methodological contribution---pignistic transform matching
for controlled fusion comparison---is independent of the specific
accuracy outcome: it provides the community with a reusable framework
for any future comparison of Bayesian and belief function fusion where
the sensor model is not a fixed shared component. Whether future
experiments with PCR6, 3D grids, or interval-valued metrics overturn
the accuracy finding, the methodology remains valid and necessary.

The identified mechanism is specific to Dempster's conflict
normalization; rules designed for the high-conflict regime (PCR6,
Dubois--Prade) may not exhibit this behavior and remain untested.
Future work should validate these results on 3D occupancy grids,
test these alternative combination rules alongside open-world TBM,
and extend to non-consonant observation masses and multi-hypothesis
frames where belief functions' capacity for partial support may offer
genuine advantages. Evaluating interval-valued metrics and calibration
assessment via Expected Calibration Error and reliability diagrams
would complement the point-probability assessment presented here.

\section*{Data Availability}
\label{sec:data-availability}

All experiment code, sensor model implementations, and statistical
analysis scripts are available at
\url{[URL upon acceptance]}. The
repository includes configuration files for reproducing all experiments
reported in this paper. The Intel Research Lab dataset is publicly
available at the Radish repository~\cite{howard2003intel}. The Freiburg
Building~079 dataset is publicly available from the University of
Freiburg. All simulation experiments use fixed random seeds (42--56) and
produce deterministic results given the same configuration.

%% file: appendix.tex

\section*{Supplementary Material}

Supplementary material associated with this article contains:
statistical methodology details (TOST, effect sizes, multiplicity
correction, spatial block bootstrap, Bayes factors, power analysis,
CI verification, block-size sensitivity), supporting proofs,
convergence analysis, ablation studies ($L_{\max}$, Yager's rule,
DS regularization, sensor parameter sensitivity), and detailed
prior-work comparison.

%% file: table_ci_verification.tex
\begin{table}[ht]
\centering
\caption{Cohen's $d$ CI verification: Hedges--Olkin (1985) approximation vs.\ exact non-central $t$ inversion (Steiger, 2004). Maximum discrepancy is 0.59, confirming qualitative conclusions are unaffected.}
\label{tab:ci-verification}
\begin{tabular}{llrccc}
\toprule
Experiment & Metric & $d$ & HO 95\% CI & NCT 95\% CI & Disc. \\
\midrule
  Single-agent & Cell Accuracy & $+13.51$ & $[+8.01,\;+19.02]$ & $[+8.54,\;+18.48]$ & 0.54 \\
   & Boundary Sharpness & $+7.78$ & $[+4.58,\;+10.98]$ & $[+4.88,\;+10.67]$ & 0.31 \\
   & Brier Score & $-14.93$ & $[-21.00,\;-8.85]$ & $[-20.41,\;-9.44]$ & 0.59 \\
  Multi-robot (dyn.) & Cell Accuracy & $+6.14$ & $[+3.59,\;+8.69]$ & $[+3.83,\;+8.44]$ & 0.25 \\
   & Boundary Sharpness & $+5.96$ & $[+3.48,\;+8.44]$ & $[+3.71,\;+8.19]$ & 0.24 \\
   & Map Entropy & $-7.11$ & $[-10.04,\;-4.17]$ & $[-9.75,\;-4.45]$ & 0.29 \\
  Multi-robot (noisy) & Cell Accuracy & $+6.03$ & $[+3.52,\;+8.53]$ & $[+3.76,\;+8.29]$ & 0.25 \\
   & Boundary Sharpness & $+4.62$ & $[+2.67,\;+6.57]$ & $[+2.85,\;+6.38]$ & 0.19 \\
   & Map Entropy & $-6.32$ & $[-8.94,\;-3.70]$ & $[-8.69,\;-3.95]$ & 0.26 \\
\bottomrule
\end{tabular}
\end{table}

%% file: table_block_sensitivity.tex
\begin{table}[ht]
\centering
\footnotesize
\caption{Block size sensitivity for spatial block bootstrap CIs. Each cell shows the 95\% CI for the metric delta (Bayesian $-$ Dempster). All significance verdicts are stable across block sizes.}
\label{tab:block-sensitivity}
\begin{tabular}{lllccc}
\toprule
Dataset & Setup & Metric & $B{=}5$ (0.5\,m) & $B{=}10$ (1.0\,m) & $B{=}20$ (2.0\,m) \\
\midrule
  Intel Lab & 1-source & Cell Accuracy & $[+0.014,\;+0.021]$ & $[+0.014,\;+0.022]$ & $[+0.013,\;+0.023]$ \\
   &  & Boundary Sharpness & $[-0.006,\;+0.007]$ & $[-0.007,\;+0.007]$ & $[-0.007,\;+0.007]$ \\
   &  & Brier Score & $[-0.016,\;-0.010]$ & $[-0.016,\;-0.010]$ & $[-0.016,\;-0.009]$ \\
   & 2-way split & Cell Accuracy & $[+0.018,\;+0.024]$ & $[+0.017,\;+0.025]$ & $[+0.017,\;+0.026]$ \\
   &  & Boundary Sharpness & $[-0.010,\;+0.003]$ & $[-0.010,\;+0.003]$ & $[-0.011,\;+0.004]$ \\
   &  & Brier Score & $[-0.018,\;-0.014]$ & $[-0.019,\;-0.013]$ & $[-0.020,\;-0.013]$ \\
   & 4-way split & Cell Accuracy & $[+0.020,\;+0.025]$ & $[+0.019,\;+0.026]$ & $[+0.018,\;+0.027]$ \\
   &  & Boundary Sharpness & $[-0.012,\;+0.001]$ & $[-0.013,\;+0.002]$ & $[-0.013,\;+0.003]$ \\
   &  & Brier Score & $[-0.019,\;-0.015]$ & $[-0.020,\;-0.014]$ & $[-0.021,\;-0.014]$ \\
\midrule
  Freiburg~079 & 1-source & Cell Accuracy & $[+0.005,\;+0.009]$ & $[+0.005,\;+0.009]$ & $[+0.005,\;+0.009]$ \\
   &  & Boundary Sharpness & $[+0.005,\;+0.018]$ & $[+0.005,\;+0.018]$ & $[+0.005,\;+0.018]$ \\
   &  & Brier Score & $[-0.009,\;-0.006]$ & $[-0.009,\;-0.005]$ & $[-0.009,\;-0.005]$ \\
   & 2-way split & Cell Accuracy & $[+0.005,\;+0.009]$ & $[+0.005,\;+0.009]$ & $[+0.005,\;+0.010]$ \\
   &  & Boundary Sharpness & $[+0.005,\;+0.016]$ & $[+0.005,\;+0.016]$ & $[+0.005,\;+0.016]$ \\
   &  & Brier Score & $[-0.009,\;-0.006]$ & $[-0.009,\;-0.006]$ & $[-0.010,\;-0.005]$ \\
   & 4-way split & Cell Accuracy & $[+0.005,\;+0.009]$ & $[+0.005,\;+0.009]$ & $[+0.005,\;+0.010]$ \\
   &  & Boundary Sharpness & $[+0.004,\;+0.016]$ & $[+0.004,\;+0.016]$ & $[+0.004,\;+0.016]$ \\
   &  & Brier Score & $[-0.009,\;-0.006]$ & $[-0.009,\;-0.005]$ & $[-0.009,\;-0.005]$ \\
\bottomrule
\end{tabular}
\end{table}